\documentclass{article}

\usepackage{PRIMEarxiv}

\usepackage[utf8]{inputenc} 
\usepackage[T1]{fontenc}    
\usepackage{hyperref}       
\usepackage{url}            
\usepackage{booktabs}       
\usepackage{amsfonts}       
\usepackage{nicefrac}       
\usepackage{microtype}      
\usepackage{lipsum}
\usepackage{fancyhdr}       
\usepackage{graphicx}       
\graphicspath{{media/}}     

\usepackage{dcolumn}
\usepackage{siunitx}
\usepackage{caption}
\usepackage[dvipsnames]{xcolor}
\usepackage[utf8]{inputenc} 
\usepackage[T1]{fontenc}    
\usepackage{hyperref}       
\usepackage{url}            
\usepackage{booktabs}       
\usepackage{amsfonts}       
\usepackage{nicefrac}       
\usepackage{microtype}      
\usepackage{xcolor}         
\usepackage{comment}
\usepackage{amssymb}
\usepackage{bbm}
\usepackage{microtype}
\usepackage{bm}
\usepackage{siunitx}
\usepackage{cancel}
\usepackage{subcaption}
\usepackage{amsmath}
\usepackage{amssymb}
\usepackage{booktabs}
\usepackage{adjustbox}
\usepackage{xcolor}
\usepackage{algorithm} 
\usepackage{algorithmic}
\usepackage{amsthm}
\usepackage{tikz}
\usetikzlibrary{fit, shapes, arrows, positioning}
\usetikzlibrary{backgrounds}
\usetikzlibrary{arrows.meta}
\usepackage{wrapfig}
\usepackage{makecell}
\usepackage[authoryear]{natbib}
\newtheorem{theorem}{Theorem}[section]
\newtheorem{corollary}{Corollary}[theorem]
\newtheorem{lemma}[theorem]{Lemma}
\newtheorem{prop}{Proposition}
\usepackage[toc,page,header]{appendix}
\usepackage{minitoc}
\usepackage{lipsum}
\usepackage{longtable}
\usepackage{pifont} 
\usepackage{wasysym} 
\usepackage{graphicx} 
\usepackage{soul}

\def\code#1{\texttt{#1}}

\title{System-Aware Neural ODE Processes for Few-Shot Bayesian Optimization}

 \author{%
  Jixiang Qing \\ 
  Imperial College London\\
  \And
  Becky D Langdon \\
  Imperial College London\\
  \And
  Robert M Lee \\ 
  BASF SE \\ 
  \And
  Behrang Shafei \\ 
  BASF SE \\ 
  \And
  Mark van der Wilk \\
  University of Oxford \\
  \And
  Calvin Tsay \\
  Imperial College London \\
  \And
    Ruth Misener \\
  Imperial College London \\
}

\begin{document}


\maketitle

\begin{abstract}
We consider the problem of optimizing initial conditions and termination time in dynamical systems governed by unknown ordinary differential equations (ODEs), where evaluating different initial conditions is costly and the state's value can not be measured in real-time but only with a delay while the measuring device processes the sample. To identify the optimal conditions in limited trials, we introduce a few-shot Bayesian Optimization (BO) framework based on the system's prior information. At the core of our approach is the System-Aware Neural ODE Processes (SANODEP), an extension of Neural ODE Processes (NODEP) designed to meta-learn ODE systems from multiple trajectories using a novel context embedding block. We further develop a two-stage BO framework to effectively incorporate search space constraints, enabling efficient optimization of both initial conditions and observation timings. We conduct extensive experiments showcasing SANODEP's potential for few-shot BO within dynamical systems. We also explore SANODEP's adaptability to varying levels of prior information, highlighting the trade-off between prior flexibility and model fitting accuracy.
\end{abstract}
 \doparttoc 
 \faketableofcontents 
 
\section{Introduction} 
This paper studies a widely encountered, yet less investigated, problem: optimizing the initial conditions and termination time in unknown dynamical systems where evaluations are computationally expensive. We assume that the primary evaluation cost comes from switching initial conditions and wish to use as few trajectories as possible. %

This issue is prevalent in multiple fields, including biotechnology, epidemiology, ecology, and chemistry. For instance, consider the optimization of a (bio)chemical reactor. Here, the objective is to determine the optimal “recipe,” i.e., the set of initial reactant concentrations and reaction times that maximize yield, enhance selectivity, and/or minimize waste \citep{taylor2023brief, schoepfer2024cost, schilter2024combining,thebelt2022maximizing}. Specifically, the high costs associated with changing reactants for multiple experimental runs highlight the need for developing efficient optimization algorithms~\citep{paulson2024bayesian}.

Bayesian Optimization (BO) \citep{frazier2018tutorial, garnett2023bayesian} is a well-established method for optimizing expensive-to-evaluate black-box objective functions. It relies on probabilistic surrogate models built from limited function evaluations to guide the optimization efficiently. However, standard Gaussian Processes (GPs) \citep{williams2006gaussian} (the \textit{de facto} probabilistic surrogate model in BO) with conventional kernels do not capture the dynamics of these systems effectively. Recent efforts to introduce Bayesian ODE models aim to incorporate suitable assumptions; however, their inference often involves time-consuming computations \citep{heinonen2018learning, dandekar2020bayesian}, making them impractical for a time-sensitive optimization scenario. Some Bayesian ODE models attempt to alleviate computational burdens by leveraging crude approximation inference \citep{ott2023uncertainty} or approximating numerical integrators \citep{hegde2022variational, ensinger2024exact}. While these approaches reduce computation time, the necessary approximations may degrade model performance. Consequently, the widely encountered yet under-investigated problem of performing BO in dynamical systems remains open, primarily due to the lack of an appropriate model. 

This work attempts to tackle this optimization problem by conceptualizing it as a \textit{few-shot} optimization, through leveraging \textit{prior information} about dynamical systems to formulate the \textit{meta tasks} to train a learning model, which is then able to adapt to new problems with very few data (a.k.a., shots) and served as our BO surrogate model. To achieve this, we focus on the Neural ODE Process (NODEP) \citep{norcliffe2021neural} as our learning model due to several useful model properties. First, NODEP combines Neural ODEs \citep{chen2018neural}—emphasizing the dynamical system perspective that may provide a more informative inductive bias—with Neural Processes (NP) \citep{garnelo2018neural}, enhancing the meta-learning aspect for few-shot optimization. Second, meta-learning endows NODEP with fast adaptation to new observations, mitigating the potential concern of model training time with incoming measurement results. Third, NODEP enables continuous gradient-based optimization over time, compared with discrete-time meta-learn models (e.g., \cite{foong2020meta}), integrating seamlessly within the BO framework. Finally, compared to other continuous-time models involving computationally intensive operations (e.g., time domain attention \citep{chen2024contiformer}), NODEP is computationally efficient, making it suitable for BO.

However, building a meta-learning-based few-shot BO framework for unknown dynamical systems is not a straightforward downstream application. An obstacle is that NODEP's functionality is limited by its original learning objective: predicting a single \textit{trajectory} state value given \textit{context} data for that trajectory. For initial condition optimization of unknown ODE systems, it is essential to learn to be aware of the underlying governing systems using context data consisting of one or several different trajectory observations, as this benefits for forecasting trajectories starting from arbitrary initial conditions. Moreover, no BO frameworks are proposed for such problem settings.
\begin{figure}[t]
    \centering
    \includegraphics[width=0.9\linewidth]{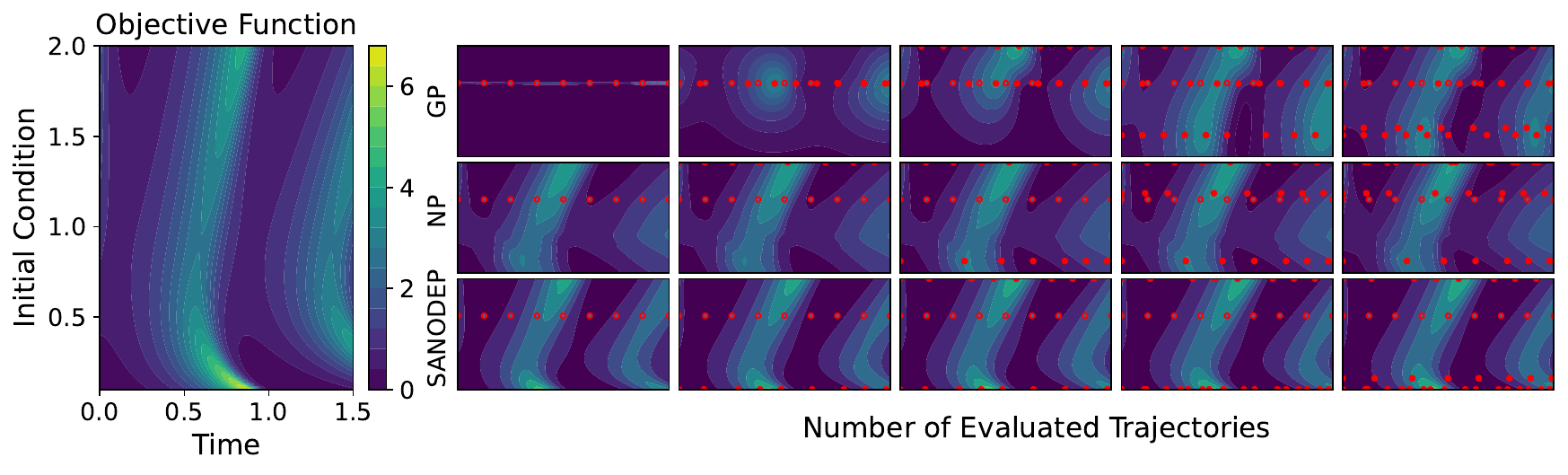}
    \caption{Illustration of meta-learning-based few-shot Bayesian Optimization (BO) with time-delay constraints (detailed in Section.~\ref{Sec:td_opt}) on the Lotka-Volterra (LV) system on different models: Gaussian Process (GP), Neural Process (NP) and System Aware Neural ODE Process (SANODEP). We start with one trajectory (marked as {\color{red} $\circ$}). Each trajectory ensures that adjacent observations respect a minimum time delay $\Delta t$ constraint. The figure clearly shows that after evaluating just one trajectory (the first, left-most column of three rows),  the meta-learned SANODEP model already resembles the original system more closely than the non-meta-learned GP, demonstrating its efficiency in guiding the search toward promising areas.}
    \label{fig:sequential_opt_illu}
\end{figure}

In line with our optimization requirements, we introduce the System-Aware Neural ODE Process (SANODEP), a generalization of NODEP in terms of meta-learning of ODE systems to plug-in our newly developed BO framework for few-shot BO in dynamical systems. Fig.~\ref{fig:sequential_opt_illu} shows, through example, SANODEP outcomes in comparison to other approaches. Our contributions are as follows:

\begin{enumerate}
\item We propose a novel context embedding mechanism enabling SANODEP to meta-learn from batch of trajectories with minimal model structure adjustments based on NODEP. 
\item We developed a model-agnostic, time-delay constraint process BO framework for optimizing initial conditions and termination time in dynamical systems. 
\item We compare SANODEP and other meta-learned models under our BO framework, validating the benefits of the ODE-aware model structure and demonstrating the effectiveness of the few-shot BO framework's model agnosticity property. 
\item We investigate how the different levels of prior information can be utilized for SANODEP. Strong prior information can enable a physically-informed model structure with extended loss considering parameter inference, enabling a novel \textit{few-shot parameter estimation} functionality, while weak prior information may still be useful through a properly designed \textit{task distribution}, albeit at a detriment to model fitting capability. 
\end{enumerate}

The rest of the paper is organized as follows: In Section~\ref{Sec: PS_BG}, we describe the preliminaries. Section~\ref{Sec:SANODEP} introduces the System-Aware Neural ODE Processes (SANODEP). In Section~\ref{Sec:td_opt}, we develop the optimization framework based on SANODEP specific to our optimization problem. Section~\ref{Sec:RelatedWork} discusses related work. Section~\ref{Sec: exp} presents the numerical experiments we conducted. In Section~\ref{Sec: PriorInfo}, we investigate the impact of different levels of prior information.

\section{Problem Statement and Background}\label{Sec: PS_BG}
\subsection{State Optimization}
Consider a dynamical system whose evolution is described by the following ordinary differential equation (ODE) system in the time domain, denoted as $t$:
\begin{equation}
\begin{aligned}
&\left\{
\begin{array}{ll}
\frac{d \boldsymbol{x}}{dt} = \boldsymbol{f}(\boldsymbol{x}, t) \\
\boldsymbol{x}(t=t_0) := \boldsymbol{x}_0 
\end{array}
\right.
\end{aligned}
\end{equation}
where $\boldsymbol{x}(t) \in \mathcal{X}_{\boldsymbol{x}} \subset \mathbb{R}^{d_{\boldsymbol{x}}}$ represents the state value of the system at time $t$, $\boldsymbol{f}(\cdot): \mathbb{R}^{d_{\boldsymbol{x}}} \times \mathbb{R} \rightarrow \mathbb{R}^{d_{\boldsymbol{x}}}$ is an unknown function (\textit{vector field}) governing the system dynamics, and $\boldsymbol{x}_0 \in \mathcal{X}_0 \subseteq \mathcal{X}_{\boldsymbol{x}} \subset \mathbb{R}^{d_{\boldsymbol{x}}}$ is the initial system state. Define $\boldsymbol{f}_{evolve}(\boldsymbol{x}_0, t, \boldsymbol{f}): \mathcal{X}_0 \times \tau \times (\mathbb{R}^{d_{\boldsymbol{x}}} \times \mathbb{R} \rightarrow \mathbb{R}^{d_{\boldsymbol{x}}})\rightarrow \mathbb{R}^{d_{\boldsymbol{x}}}$, where $\tau \in [t_0, t_{max}]$, to obtain the system state at time $t$ from an initial condition $\boldsymbol{x}_0$ as: $\boldsymbol{f}_{evolve}(\boldsymbol{x}_0, t, \boldsymbol{f})= \boldsymbol{x}_0 + \int_{t_0}^{t}\boldsymbol{f}(\boldsymbol{x}(h), h)dh$. Let $g(\cdot): \mathbb{R}^{d_{\boldsymbol{x}}} \rightarrow \mathbb{R}$ be a practitioner-specified (known) function that aggregates the state values into a scalar, we consider the following multi-objective  problem:
\begin{equation}
    \mathrm{maximize}_{\{t, \boldsymbol{x}_0\} \in \tau \times \mathcal{X}_0} \quad g\left(\boldsymbol{f}_{evolve}(\boldsymbol{x}_0, t, \boldsymbol{f})\right), - t
\label{Eq: multi_obj}
\end{equation}
The goal is to identify the initial conditions $\boldsymbol{x}_0$ along with the corresponding evolution termination time $t$ that provide the optimal trade-off between the objective $g$ applied to the state values $\boldsymbol{x}(t)$ and the amount of time to reach this state. Simply put, we wish to maximize objective $g$ early.  

\subsection{Neural ODE Processes (NODEP)}\label{Sec:NODEP_intro}
NODEP \citep{norcliffe2021neural} is a latent variable-based Bayesian meta-learning model with generative processes that can be summarized as follows: given a set of $N$ observations of state values at different times from a single trajectory, represented by context set $\mathbb{C}:=\{(t_i^\mathbb{C}, \boldsymbol{x}_i^\mathbb{C})\}_{i=1}^N$, NODEP assumes the conditional prediction has been generated from a latent controlled ODE of state dimensionality $d_{\boldsymbol{\ell}}$. The stochasticity of the model is induced by stochastic latent initial condition $L_0$ and the stochastic dynamics representation, which in practice is implemented via a 
time-invariant control $L_D$. These two terms are obtained through first encoding the context elements $\phi_r([t_i^\mathbb{C}, \boldsymbol{x}_i^\mathbb{C}])$ and then applying average pooling \citep{zaheer2017deep} to produce a single representation vector $\boldsymbol{r} = \frac{1}{N}\sum_{i=1}^N \phi_r([t_i^\mathbb{C}, \boldsymbol{x}_i^\mathbb{C}])$, which is then mapped to corresponding distributions $L_0$ and $L_D$. Once we sample the latent initial condition and the trajectory dynamics representation $\boldsymbol{u}$\footnote{With slight abuse of notation, we follow the common notation in control to use $\boldsymbol{u}$ (instead of $\boldsymbol{d}$) to represent the realization of time-invariant control term $L_D$. }, the evolution of the latent ODE is:
\begin{equation}
    \boldsymbol{l}(t)=\boldsymbol{f}_{evolve}(\boldsymbol{l}_0, t, \boldsymbol{f}_{nn}(\boldsymbol{u}, \theta_{ode})) := \boldsymbol{l}(t_{0}) + \int_{t_0}^{t}\boldsymbol{f}_{nn}\left(\boldsymbol{l}(h), \boldsymbol{u}, h, \theta_{ode}\right) dh
\label{Eq: nodep_solve}
\end{equation}
\noindent Here $\boldsymbol{f}_{nn}$ represents the vector field as a neural network \footnote{
We note that the latent ODE $\boldsymbol{f}_{nn}$ in Eq.~\ref{Eq: nodep_solve} does not have to align with the form of the unknown $\boldsymbol{f}$ in Eq.~\ref{Eq: multi_obj}. However, in section.~\ref{Sec: PriorInfo}, we also show that if explicitly know the form of $\boldsymbol{f}$, this is still possible to be leveraged as $\boldsymbol{f}_{nn}$.

} parameterized by $\theta_{ode}$. We compute $\boldsymbol{f}_{evolve}$ using numerical ODE solvers. To predict the state values at time $t$, after solving the latent ODE, the latent state $\boldsymbol{l}(t) \in \mathbb{R}^{d_{\boldsymbol{l}}}$ can be decoded back to real state predictions through a decoder $p(\boldsymbol{x}(t) | \boldsymbol{l}(t), t)$. Thus, the overall joint probability of the model can be represented as:
\begin{equation}
\begin{aligned}
& p\left(\boldsymbol{u}, \boldsymbol{l}_0, \boldsymbol{x}| \mathbb{C}, t \right)  = 
p\left(\boldsymbol{x} | \boldsymbol{f}_{evolve}(\boldsymbol{l}_0, t, \boldsymbol{f}_{nn}(\boldsymbol{u}, \theta)), t\right) p(\boldsymbol{u}| \mathbb{C} ) p(\boldsymbol{l}_0| \mathbb{C} ) 
\end{aligned}
\end{equation}
The latent state's dimensionality $d_{\boldsymbol{l}}$ is typically chosen to be larger than $d_{\boldsymbol{x}}$ as additional state variables benefit the model flexibility \citep{dupont2019augmented} and enable learning of higher-order dynamics implicitly \citep{norcliffe2020second}. 

Given a context set $\mathbb{C}$ and a target set $\mathbb{T}:=\{(t_k^{\mathbb{T}}, \boldsymbol{x}_k^{\mathbb{T}})\}_{k=1}^{J}$ \footnote{We follow the \cite{garnelo2018neural} and \cite{norcliffe2021neural} convention by assuming $\mathbb{T}$ is a superset of $\mathbb{C}$.}, to calculate the log-likelihood on $\mathbb{T}$ for inference, the intractable posterior of $p\left(\boldsymbol{l}_0| \mathbb{C} \right)p(\boldsymbol{u} \vert \mathbb{C})$ has been approximated through the mean field approximation $q(\boldsymbol{\ell}_0\vert \mathbb{T})q(\boldsymbol{u} \vert \mathbb{T})$, which in practice has been implemented through encoder $q_{L}\left(\boldsymbol{l}_0| \mathbb{C}\right)$, $q_{D}\left(\boldsymbol{u}| \mathbb{C} \right)$ in
an amortized fashion, eventually leading to the following evidence lower bound (ELBO): 
\begin{equation}
\begin{aligned}
& \text{log}\ p\left(\{(x_k^\mathbb{T})\}_{k=1}^J | \{(t_i^\mathbb{C}, \boldsymbol{x}_i^\mathbb{C})\}_{i=1}^N,  \boldsymbol{T}^\mathbb{T}  \right)  \\ & \approx \mathbb{E}_{q(\boldsymbol{\ell}_0, \boldsymbol{u}\vert \mathbb{T})} \left[\sum_{k=1}^J\text{log}\ p\left(\boldsymbol{x}_k | \boldsymbol{u}, \boldsymbol{l}_0, t_k\right)\right]  - \text{KL}\left[q_{L}\left(\boldsymbol{l}_0| \mathbb{T}\right) | | q_{L}\left(\boldsymbol{l}_0| \mathbb{C}\right)\right]  - \text{KL}\left[ q_{D}\left(\boldsymbol{u} | \mathbb{T} \right) | | q_{D}\left(\boldsymbol{u}| \mathbb{C} \right)\right] 
\end{aligned}
\label{Eq: loss_func_of_NODEP}
\end{equation}
\noindent where $\boldsymbol{T}^\mathbb{T}:=\{(t_k^{\mathbb{T}})\}_{k=1}^J$ represents the (irregularly sampled) target times\footnote{We denote context trajectory data either as a set of pairs $\{(t_i^\mathbb{C}, \boldsymbol{x}_i^\mathbb{C})\}_{i=1}^N$ or as vectors $\{\boldsymbol{T}^{\mathbb{C}}, \boldsymbol{X}^{\mathbb{C}}\}$, where $\boldsymbol{T}^{\mathbb{C}}=[t_1^{\mathbb{C}}, ..., t_N^{\mathbb{C}}]^\mathrm{T}$,   $\boldsymbol{X}^{\mathbb{C}}=[\boldsymbol{x}_1^{\mathbb{C}}, ..., \boldsymbol{x}_N^{\mathbb{C}}]^{\mathrm{T}}$. We use similar notation $\{\boldsymbol{T}^{\mathbb{T}}, \boldsymbol{X}^{\mathbb{T}}\}$ for the target set.}. NODEP has competitive performance in predicting single trajectory, and we therefore consider extending its functionality to meta-learn ODE system distributions for our subsequent optimization purpose.

\section{System Aware Neural ODE Process (SANODEP)}\label{Sec:SANODEP}

\begin{wrapfigure}[23]{r}{0.45\textwidth} 
\vspace{-10pt}
\hspace{-1pt}
\centering
\begin{tikzpicture}[node distance=0.7cm and 0.4cm, auto] 
  \tikzset{
    var/.style={circle, thick, draw, minimum size=1cm},
    latent/.style={circle, thick, draw, fill=gray!25, minimum size=1cm}
  }

  \node[latent] (tc_not_j) {$\boldsymbol{T}^{\mathbb{C}}$};
  \node[latent, right=of tc_not_j] (xc_not_j) {$\boldsymbol{X}^{{\mathbb{C}}}$};
  \node[latent, right=of xc_not_j] (tc) {$\boldsymbol{T}_{new}^{\mathbb{C}}$};
  \node[latent, right=of tc] (xc) {$\boldsymbol{X}_{new}^{{\mathbb{C}}}$};

  \node[var, below=of tc] (l0) {${\boldsymbol{L}_0}_{new}$};
  \node[var, right=of l0] (dsys) {$D_{sys}$};
  \node[latent, below=of l0] (tt) {$\boldsymbol{T}_{new}^{{\mathbb{T}}}$};
  \node[latent, right=of tt] (xt) {$\boldsymbol{X}_{new}^{{\mathbb{T}}}$};

  \draw[-{Latex[length=2mm]}] (xc) -- (l0);
  \draw[-{Latex[length=2mm]}] (tc) -- (l0);
  \draw[-{Latex[length=2mm]}] (xc) -- (dsys);
  \draw[-{Latex[length=2mm]}] (tc_not_j) -- (dsys);
  \draw[-{Latex[length=2mm]}] (xc_not_j) -- (dsys);
  \draw[-{Latex[length=2mm]}] (tc) -- (dsys);
\draw[-{Latex[length=2mm]}] ([xshift=0.5mm, yshift=0.5mm]l0.south east) -- ([xshift=0.5mm, yshift=0.5mm]xt.north west);
  \draw[-{Latex[length=2mm]}] (l0) -- (dsys);
  \draw[-{Latex[length=2mm]}] ([xshift=-0.3mm]dsys.south) -- ([xshift=-0.5mm]xt.north);
  \draw[-{Latex[length=2mm]}, dashed] ([xshift=0.5mm]xt.north) -- ([xshift=0.7mm]dsys.south);
    \draw[-{Latex[length=2mm]}, dashed] (tt)--(dsys);
  \draw [-{Latex[length=2mm]}] (tt) -- (xt);
  \draw[-{Latex[length=2mm]}, dashed] ([xshift=-0.5mm, yshift=-0.5mm]xt.north west) -- ([xshift=-0.5mm, yshift=-0.5mm]l0.south east);
    \draw[-{Latex[length=2mm]}, dashed] (tt)--(l0);
  \draw[thick] (-0.7,-0.8) rectangle (5.5, 1.0);
  \draw[thick] (2.1, -4.8) rectangle (5.5, -3.1);
  \node at (-0.5, 0.7) {$\mathbb{C}$};
  \node at (2.3,-3.4) {$\mathbb{T}$};
\end{tikzpicture}

\caption{Graphical Model of SANODEP, the model predicts any time point in the new trajectory by knowing both observations from new trajectories and from past trajectories. Depending on whether the $X_{new}^{{\mathbb{C}}}$ and $T_{new}^{{\mathbb{C}}}$ consists of more than the initial condition, the model focuses on forecasting or interpolating tasks. The solid and dashed lines represent the generative and inference processes, respectively.}
\end{wrapfigure}

\textbf{Few-Shot Optimization with Prior Information} Given that evaluating $\boldsymbol{f}_{evolve}(\boldsymbol{x}_0, t, \boldsymbol{f})$ is limited by cost to a small number of different initial conditions. To enable fast adaptation with few evaluations, we take a meta-learning approach by assuming that $\boldsymbol{f}$ is a realization of a stochastic function $\mathcal{F}$, and we can access data from its different realizations to formulate the \textit{meta-task distributions}, 

 More precisely, consider a random function $\mathcal{F}: \mathcal{X}_{0} \times \tau \rightarrow \mathbb{R}^{d_{\boldsymbol{x}}}$ representing the distribution of dynamical systems. From a specific realization of $\mathcal{F}$, suppose we have observed $M$ distinct trajectories, the context set then encompasses observations from all $M$ trajectories, denoted as $\mathcal{T}^\mathbb{C} = \{\mathcal{T}_1^\mathbb{C}, ..., \mathcal{T}_M^\mathbb{C}\}$. Each trajectory, labeled as $\mathcal{T}_l^\mathbb{C}$, includes its own set of $N_l$ context observations: $\mathcal{T}_l^\mathbb{C} = \{({t_{l}}_i^{\mathbb{C}}, {\boldsymbol{x}_{l}}_i^{\mathbb{C}})\}_{i=1}^{N_l}$. Furthermore, consider a \textit{new} trajectory $\mathcal{T}_{new}$. After observing additional context observations on this trajectory, $\mathcal{T}_{new}^{\mathbb{C}}$ (e.g., the initial condition of this new trajectory), we can define the extended context set as $\mathbb{C} = \mathcal{T}^\mathbb{C} \cup \mathcal{T}_{new}^{\mathbb{C}}$. Additionally, the target set is defined as $\mathbb{T} = \mathcal{T}^\mathbb{C} \cup \mathcal{T}_{new}^\mathbb{T}$\footnote{Note that again $\mathcal{T}_{new}^\mathbb{C} \subset \mathcal{T}_{new}^\mathbb{T}$.}, a visual elaboration is provided in Fig.~\ref{fig:interpolating_scenario_illustration} in Appendix \ref{app: pb_illu_nomenclature}.

It is straightforward to see why NODEP is suboptimal in such scenario as it makes predictions based on single time series ($\mathbb{C}=\mathcal{T}_{new}^\mathbb{C}$, $\mathbb{T}=\mathcal{T}_{new}^\mathbb{T}$), unable to leverage $M$ trajectorys' information. Below, we demonstrate how SANODEP is efficiently enabled through the set-based representation \citep{zaheer2017deep}.

\subsection{Set-based Dynamical System Representation}

Similar to the latent variable $\boldsymbol{u}$ (Eq.\ \ref{Eq: nodep_solve}) capturing trajectory dynamics in NODEP, we adapt a latent variable $D_{sys} \sim q(\boldsymbol{u}_{sys}| \mathbb{C})$  in SANODEP, effectively replacing $\boldsymbol{u}$ in the model structure but with an enhanced conditioning on context observations from $M+1$ trajectories, to capture the dynamical systems properties. Follow this mechanism, we will have a feature extraction for \textbf{multi-start multivariate irregular time series}, in addition with an efficiency requirement for our optimization purposes. 

Such questions are less considered in contemporary time series models, as existing approaches \citep{shukla2021multi, schirmer2022modeling} are mainly developed for single initial condition start trajectories. Specifically for multi-start scenario, \cite{jiang2022sequential} proposes first extracting a trajectory-wise aggregated feature vector, and then averaging feature vectors as a final representation. However, both extraction are implemented through a convolution operation, which poses challenges with irregular time series without additional modifications.

A straightforward thinking would be still using average pooling across all context elements as a set-based representation. While the Picard-Lindelöf theorem \citep{lindelof1894application} guarantees the uniqueness of state values in an initial value problem when $\boldsymbol{f}(\boldsymbol{x}, t)$ is Lipschitz continuous, in case when only part of the underline system states can be measured as $\boldsymbol{x}(t)$, different trajectories might still share identical $\boldsymbol{x}(t)$ values, leading to identifiability issues (duplicate context elements come from different initial conditions).  Consequently, we propose an augmented sets-based approach: we augment each observation with its corresponding initial condition: $({t}_i, {\boldsymbol{x}}_0, {\boldsymbol{x}}_i)$,  to enhance the model's ability to differentiate between trajectories that might otherwise appear identical. Then we perform the average pooling on the flattened context set to obtain the context representation $\boldsymbol{r}_{sys}= \frac{1}{\sum_{l=1}^{M+1}N_l}\sum_{l=1}^{M+1}\sum_{i=1}^{N_l}{\phi_r}_{sys}([{t_l}_i^{\mathbb{C}}, {\boldsymbol{x}_l}_0^{\mathbb{C}}, {\boldsymbol{x}_l}_i^{\mathbb{C}}])$, which is then mapped to a distribution representing possible system realizations from $\mathcal{F}$ that has generated $\mathbb{C}$. Aside from the theoretical intuition, we also empirically compare the average pooling without initial condition augmentation in the section.~\ref{Sec:train_model_comp}, and find that the augmentation leads to more robust performance even  
$\mathcal{F}$ only represents first-order systems with fully observable states.

For the rest of the model, SANODEP follows the NODEP structure, hence requiring minimal adjustment, only with additional care on activation function choices to enable differentiability (Appendix.~\ref{App: differentiability}). Appendix~\ref{App: model_structure} provides a detailed description of the model structure. 

\subsection{Bi-scenario Loss Function} 
In line with the principles of episode learning \citep{vinyals2016matching}, SANODEP's training is structured through multiple episodes. Motivated by our optimization problem that will be detailed in Section.~\ref{Sec:td_opt}, the design of each episode's problem ensures good model performance under two primary scenarios: 
\begin{enumerate}
    \item \textbf{Forecasting}: Using $M$ context trajectories $\mathcal{T}^\mathbb{C}$, the model predicts future state values ${\boldsymbol{X}_{new}^{\mathbb{T}}}$ for a new trajectory initiated from $\mathcal{T}_{new}^{\mathbb{C}} = \{({t_{new}}_0^{\mathbb{C}}, {{\boldsymbol{x}_{new}}_{0}^{\mathbb{C}}})\}$ at designated target times ${\boldsymbol{T}_{new}^{\mathbb{T}}}$.
    \item \textbf{Interpolating}: From the same $M$ trajectories, the model interpolates and extrapolates state values for a new trajectory that already includes $K>1$ observations: $\mathcal{T}_{new}^{\mathbb{C}} = \{({t_{new}}_i^{\mathbb{C}}, {{\boldsymbol{x}_{new}}_i^{\mathbb{C}}})\}_{i=0}^K$, predicting the states at times ${\boldsymbol{T}_{new}^{\mathbb{T}}}$.
\end{enumerate}

Both scenarios can be considered under a bi-scenario loss function, designed to enhance model's accuracy for predicting new trajectory states ${\boldsymbol{X}_{new}^{\mathbb{T}}}$ at ${\boldsymbol{T}_{new}^{\mathbb{T}}}$:
\begin{equation}
\begin{aligned}
    \mathcal{L}_{\theta}  = & \mathbb{E}_{\boldsymbol{f}\sim \mathcal{F}, M, \mathcal{T}^\mathbb{C}, \mathcal{T}_{new}^{\mathbb{T}}, \mathbbm{1}_{forecast}, \mathcal{T}_{new}^{\mathbb{C}}(\mathbbm{1}_{forecast})} \text{log}\ p_{\theta}\left({{\boldsymbol{X}_{new}^{\mathbb{T}}}} | \mathcal{T}^\mathbb{C} \cup \mathcal{T}_{new}^{\mathbb{C}}(\mathbbm{1}_{forecast}), {\boldsymbol{T}_{{new}}^{\mathbb{T}}}\right)
\label{Eq: multi_scenario_loss_fn}
\end{aligned}
\end{equation}
\noindent where $p_{\theta}(\cdot)$ represents the SANODEP prediction when parameterized by $\theta$.  

During training, once a dynamical system realization $\boldsymbol{f}$ from $\mathcal{F}$ has been drawn,  we randomly sample the number of context trajectories $M$ that we have observed, where the number of context elements, the initial condition of the trajectory and the observation time are all sampled. For the new trajectory to be predicted, besides sampling the target set $\{\boldsymbol{T}_{new}^{\mathbb{T}}, \boldsymbol{X}_{new}^{\mathbb{T}}\}$, 
a Bernoulli indicator $\mathbbm{1}_{forecast}\sim \text{Bernoulli}(\lambda)$ parameterize by $\lambda$ is sampled to determine whether the episode will address a forecasting or interpolating scenario. This indicator directly influences the part of the context set sampled from the new trajectory\footnote{For notation simplicity, we always use subscript $0$ to represent the initial condition of a trajectory.}: 
\begin{equation}
\mathcal{T}_{new}^{\mathbb{C}}(\mathbbm{1}_{forecast}) = 
\begin{cases} 
\{{({t_{new}}_{0}^{\mathbb{C}}}, {{\boldsymbol{x}_{new}}_{0}^{\mathbb{C}}})\} & \text{if } \mathbbm{1}_{forecast} = 1, \\
 \{({t_{new}}_{i}^{\mathbb{C}}, {{\boldsymbol{x}_{new}}_{i}^{\mathbb{C}}})\}_{i=0}^K & \text{if } \mathbbm{1}_{forecast}  = 0.
\end{cases}
\end{equation}

For each system realization, we train in a mini-batch way by making prediction on a batch of different new trajectories.

The intractable log-likelihood in Eq.~\ref{Eq: multi_scenario_loss_fn} can be approximated via the following evidence lower bound (ELBO): 
\begin{equation}
\begin{aligned}
& \log p_{\theta}\left(\boldsymbol{X}_{new}^{\mathbb{T}} \mid \mathcal{T}^\mathbb{C} \cup \mathcal{T}_{new}^{\mathbb{C}}(\mathbbm{1}_{forecast}), \boldsymbol{T}_{new}^{\mathbb{T}} \right) \\
& \approx \mathbb{E}_{q\left(\boldsymbol{u}_{sys} \mid \mathcal{T}^\mathbb{C} \cup \mathcal{T}_{new}^{\mathbb{C}}(\mathbbm{1}_{forecast}) \cup \mathcal{T}_{new}^{\mathbb{T}}\right)q\left(\boldsymbol{L}_{0_{new}}^{\mathbb{T}} \mid \left(t_0^{\mathbb{C}}, \boldsymbol{x}_{0_{new}}^{\mathbb{C}}\right)\right)}  \log p\left(\boldsymbol{X}_{new}^{\mathbb{T}} \mid \boldsymbol{T}_{new}^{\mathbb{T}}, \boldsymbol{u}_{sys}, \boldsymbol{L}_{0_{new}}^{\mathbb{T}}\right) \\
& \quad - \text{KL}\left[q\left(\boldsymbol{u}_{sys} \mid \mathcal{T}^\mathbb{C} \cup \mathcal{T}_{new}^{\mathbb{C}}(\mathbbm{1}_{forecast}) \cup \mathcal{T}_{new}^{\mathbb{T}}\right) \parallel q\left(\boldsymbol{u}_{sys} \mid \mathcal{T}^\mathbb{C} \cup \mathcal{T}_{new}^{\mathbb{C}}(\mathbbm{1}_{forecast}) \right)\right] \\
& \quad - \text{KL}\left[q\left(\boldsymbol{L}_{0_{new}}^{\mathbb{T}} \mid \left(t_{0_{new}}^{\mathbb{C}}, \boldsymbol{x}_{0_{new}}^{\mathbb{C}}\right)\right) \parallel p(\boldsymbol{L}_{0_{new}}^{\mathbb{T}})\right]
\end{aligned}
\label{Eq: ELBO_form}
\end{equation}

where $q\left(\boldsymbol{u}_{sys} \mid \mathcal{T}^\mathbb{C} \cup \mathcal{T}_{new}^{\mathbb{C}}(\mathbbm{1}_{forecast}) \cup \mathcal{T}_{new}^{\mathbb{T}}\right)$ and $q\left(\boldsymbol{L}_{0_{new}}^{\mathbb{T}} \mid \left(t_{0_{new}}^{\mathbb{C}}, \boldsymbol{x}_{0_{new}}^{\mathbb{C}}\right)\right)$ has been obtained through the encoder in an amortized fashion, the prior $p(\boldsymbol{L}_{0_{new}}^{\mathbb{T}})$ is isotropic Gaussian. Appendix.~\ref{Sec: loss_derive} derives the ELBO and provides implementation details of the model inference procedure.

\section{Time Delay Constraint Process Bayesian Optimization}\label{Sec:td_opt}

Through maximizing the ELBO for $\theta$, we can obtained SANODEP's predictive distribution $p_{\theta}(\boldsymbol{X}^{\mathbb{T}}\vert \mathbb{C}, \boldsymbol{T}^{\mathbb{T}}, \boldsymbol{x}_0)$ for batch of time $\boldsymbol{T}^{\mathbb{T}}=\{t_1, ..., t_N\}$ at specified initial condition $\boldsymbol{x}_0$, which is sufficient for few-shot learning tasks. Specifically for our practically motivated termination time optimization problem that incorporates one additional time delay constraint, we further propose an optimization framework in this section, leveraging SANODEP for few-shot BO in ODE and benchmark in Section.~\ref{sec: meta_bo_exp}. 


\textbf{Minimum Observation Delay Constraint} When optimizing Eq.~\ref{Eq: multi_obj}, we assume that, while one can observe state values at any chosen time $t$ on a specific initial condition $\boldsymbol{x}_0$, as illustrated in Fig.~\ref{fig:sequential_opt_illu}, the next observation can only commence after a fixed known time duration $\Delta t$. In practice, these delays stem from the need to sequentially conduct separate, smaller experiments for \textit{state value} measurements, each requiring a known period to complete before the next can begin.

The optimization framework consists of the following two steps: 
\newline\textbf{Initial Condition Identification}  identifies the optimal initial conditions necessary for starting experiments. The optimality of the initial conditions is defined as achieving the maximum expected reward after the completion of observations starting at this location.  Inspired by the batch strategy to achieve a similar \textit{non-myopic} objective \citep{gonzalez2016glasses, jiang2020binoculars}, we propose an \textbf{adaptive batch size} based optimization strategy for the initial condition identification:
\begin{equation}
\begin{aligned}
    & \mathrm{maximize}_{\boldsymbol{x}_0 \in \mathcal{X}_0, \{t_1, t_2, ..., t_N \in \tau\}, N \in \mathcal{N}_{opt}}\ \alpha\left(\boldsymbol{x}_0, t_1, ..., t_N,  p(\boldsymbol{X}^{\mathbb{T}}\vert \mathbb{C}, \boldsymbol{T}^{\mathbb{T}}, \boldsymbol{x}_0)\right) \\ 
    & s. t.\  \forall i \in \{1, ..., N\}: t_{i}-t_{i-1} \geq \Delta t 
\end{aligned}
\label{Eq: joint_acq_opt_form}
\end{equation}
\noindent where $\alpha$ is the batch acquisition function to be maximized, the search space of the total number of observation queries $\mathcal{N}_{opt} := \{1, 2, ..., N_{max}\}$, $\boldsymbol{T}^{\mathbb{T}}$  now consists of  $\{t_1, ..., t_N\}$ and the maximum trajectory-wise observations number is $N_{max}:=\lfloor (t_{max} - t_0) / \Delta t \rfloor$, we omit the dependence of $p_{\theta}(\boldsymbol{x}\vert \mathbb{C}, t, \boldsymbol{x}_0)$ in $\alpha(\cdot)$ for brevity thereafter.  
\newline\textbf{Choose the next query time}  Suppose the initial condition $\boldsymbol{x}_0$ has been chosen, and our last observation query made is at time $t_n$, when still have query opportunity (i.e., $t_{max} > t_n + \Delta t$), redefine the maximum remaining trajectory observation $N_{max}$ as  $\lfloor \frac{t_{max} - t_n}{\Delta t} \rfloor$, the batch size search space $\mathcal{N}_{opt}$ and time search space $\tau = [t_n + \Delta t, t_{max}]$, we choose the next query time recurrently via the following optimization problem:   

\begin{equation}
\begin{aligned}
    & \mathrm{maximize}_{\{t_1, t_2, ..., t_{N} \in \tau\}, N \in \mathcal{N}_{opt}}\ \alpha\left(\boldsymbol{x}_0, t_1, ..., t_N \right) \\ 
    & s. t.\  \forall i \in \{1, ..., N\}: t_{i}-t_{i-1} \geq \Delta t 
\end{aligned}
\label{Eq: last_acq_opt_form}
\end{equation}
\newline\textbf{Search Space Reduction} The integer variable $N$'s search space $\mathcal{N}_{opt}$, though one dimensional, can be cumbersome to optimize in practice. However, for the batch expected hypervolume improvement acquisition function (\texttt{qEHVI}) \citep{daulton2020differentiable} that we will use as $\alpha$, the search space can be reduced: 

\begin{prop}
For acquisition functions defined as $\alpha() = \mathbb{E}_{p(\boldsymbol{X}^{\mathbb{T}} \mid \mathbb{C}, \boldsymbol{T}^{\mathbb{T}}, \boldsymbol{x}_0)}\left[\mathrm{HVI}\left(\boldsymbol{X}^{\mathbb{T}}, \boldsymbol{T}^{\mathbb{T}}, \mathcal{F}^* \vert \mathbb{C}\right)\right]$:
\begin{align*}
\mathrm{maximize}_{\substack{\boldsymbol{x}_0 \in \mathcal{X}_0, \\ \{t_1, t_2, \dots, t_N\} \in \tau, \\ N \in \{1, \dots, N_{\text{max}}\}}} \alpha(\boldsymbol{x}_0, t_1, \dots, t_N) &= \mathrm{maximize}_{\substack{\boldsymbol{x}_0 \in \mathcal{X}_0, \\ \{t_1, t_2, \dots, t_N\} \in \tau, \\ N \in \{\lceil N_{\text{max}}/2 \rceil, \dots, N_{\text{max}}\}}} \alpha(\boldsymbol{x}_0, t_1, \dots, t_N).
\end{align*}
\label{thm: search_space_reduction}
\end{prop}
where HVI stands for \textit{hypervolume improvement} based on Pareto frontier $\mathcal{F}^*$, see Definition~2 of \cite{daulton2020differentiable}. Appendix \ref{Sec: Acq_Fn}  provides the proof and shows that the search space reduction property holds for generic acquisition functions that are monotonic w.r.t.\ set inclusions. Consequently, we define  $\mathcal{N}_{opt}$ as $[\lceil\frac{N_{max}}{2}\rceil, N_{max}]$ thereafter. 

\textbf{Optimization Framework} is outlined in Algorithm.~\ref{alg: opt_framework} in Appendix.~\ref{Sec: opt_framework}. We refer to the Appendix.~\ref{Sec: Acq_Fn} for details of how SANODEP is utilized within \texttt{qEHVI}, together with acquisition function optimizer described in Appendix.~\ref{Sec: Acq_Fn}. Finally, 
we highlight that this Bayesian Optimization framework is agnostic to the model choice of $p_{\theta}(\boldsymbol{X}^{\mathbb{T}} \mid \mathbb{C}, \boldsymbol{T}^{\mathbb{T}}, \boldsymbol{x}_0)$, as we will compare in Section.~\ref{Sec: exp}, both NP and GP (as a non-meta-learn model) can be used.

\section{Related Works}\label{Sec:RelatedWork} \label{sec:rel_work}
\textbf{Meta-learning of Dynamical Systems}: Our work is built upon a meta-learned continuous time-based model for state prediction in dynamical systems. \cite{singh2019sequential} extended Neural Processes \citep{garnelo2018neural} by incorporating temporal information from the perspective of the state-space model. However, their approach is limited to discrete-time systems. Following the NODEP framework, \cite{dang2023conditional} extended the model for forecasting purposes in multi-modal scenarios. Recently, \cite{jiang2022sequential} used meta-learning to tackle high-dimensional time series forecasting problems for generic sequential latent variable models, they leverage spatial-temporal convolution to extract the $\boldsymbol{u}_{sys}$, which only works on regularly spaced time steps (e.g., image frames in their use cases). However, SANODEP extracts system dynamic's feature from multiple irregularly sampled multivariate time series using the newly proposed encoder structure.  \cite{song2023towards} explored Hamiltonian representations, which are flexible enough for cross-domain generalizations. Beyond the Bayesian meta-learning paradigm, \cite{li2023metalearning} used classical gradient-based meta-learning \citep{finn2017model} to meta-learn dynamics with shared properties. \cite{auzina2024modulated} investigated the separation and modeling of both dynamic variables that influence state evolution and static variables that correspond to invariant properties, enhancing model performance.

\textbf{Process-Constrained Bayesian Optimization:} \citet{vellanki2017process} address process-constrained batch optimization where each batch shares identical constrained variables within a subspace. 
Our problem similarly treats initial conditions as constrained variables in batch optimization but differs by using a joint expected utility method instead of a greedy strategy. Additionally, our approach involves on-the-fly optimization with system evolution, leading to a novel time-constrained problem addressed by an adaptive batch BO algorithm with search space reduction. \citet{folch2022snake, folch2024transition} also consider optimization under movement (initial condition) constraints.

\textbf{Few-Shot Bayesian Optimization}  A significant amount of recent work has focused on meta-learning stochastic processes, to be used as alternatives to GP for BO. Within the (Conditional) Neural Process framework ((C)NP) \citep{garnelo2018conditional, garnelo2018neural}, Transformer Neural Processes (TNP) \citep{nguyen2022transformer} have demonstrated robust uncertainty quantification capabilities for BO. \cite{dutordoir2023neural} introduced a diffusion model-based Neural Process, providing a novel framework for modeling stochastic processes which enables joint sampling and inverse design functionality. Similar to TNP, \cite{muller2021transformers} approached the meta-learning problem as Bayesian inference, subsequently integrating a BO framework into this model \citep{muller2023pfns4bo}. We note that while a majority of the aforementioned approaches lead to promising empirical performance on BO for regular black-box functions, their backend models (if there are any) are mainly motivated by a drop-in replacement of GP or non-temporal stochastic process. A meta-learned BO method specifically for dynamical systems is, to our best of knowledge, not revealed yet.

\section{Experiments}\label{Sec: exp}
This section conducts experiments on meta-learning and few-shot BO for dynamical systems. All models are implemented using $\texttt{Flax}$ \citep{flax2020github} and are open source, available in:  \url{https://github.com/TsingQAQ/SANODEP}. The optimization framework is based on \texttt{Trieste} \citep{picheny2023trieste}. 
\subsection{Modeling Comparison} \label{Sec:train_model_comp}
We evaluate model performance comparisons on meta-learning diverse ODE system distributions with a twofold purpose. First, we investigate whether leveraging the system information can help modeling. We also investigate whether the inductive bias could bring performance benefits.

\textbf{Baseline}: To see if adding system information helps with the modeling, we compare SANODEP against NODEP under its original testing scenario (interpolation). For the second purpose, we compare SANODEP with Neural Processes (NP) \citep{garnelo2018neural}. As mentioned in Section \ref{sec:rel_work}, we were not able to compare with the meta ODE models of \cite{jiang2022sequential} as it is not been able to be applied on irregular time series. We also provide GP results for reference. For models that do not explicitly incorporate the ODE in their structure, we use the trajectory's initial condition augmented with time as its input. 


\textbf{Training} Following  \cite{norcliffe2021neural},  we treat $\mathcal{F}$ as a parametric  function of a specific \textit{kinetic model} with stochasticity induced by model parameter distributions $P$. We refer to Appendix.~\ref{sec: training_data_describe} for the model formulations, its parameter distribution $P$'s form, and the training epochs. Interpolation-only experiments use SANODEP trained using Eq.~\ref{Eq: multi_scenario_loss_fn}, setting $\lambda=0$, and NODEP trained using Eq.~\ref{Eq: loss_func_of_NODEP}. For the second experimental purpose, we train all models using Eq.~\ref{Eq: multi_scenario_loss_fn} setting $\lambda = 0$ for interpolation only training and $\lambda=0.5$ for bi-scenario training, with 5 different random seeds for model initialization.  

\textbf{Evaluation}: 
We evaluate the model performance on target data generated from systems sampled from the same $\mathcal{F}$, which is known as \textit{in-distribution} generalization evaluation. Excluding GP, each model was evaluated on $10^4$ random systems, each consisting of 100 trajectories to predict in a minibatch fashion. As it is computationally infeasible to evaluate GPs on the same scale, we used a random subset of the test set, 5,000 trajectories. Appendix.~\ref{App: learn_and_pred} provides full data generation settings. Following \cite{norcliffe2021neural}, we assess model performance in terms of Mean-Squared Error by conducting a sequential evaluation with varying numbers of context trajectories, as illustrated in Fig.~\ref{fig:sequential_model_comparison} (and  Fig.~\ref{fig:seq_model_eval_nll_with_gp} in Appendix~\ref{Sec:add_res} for Negative-log-Likelihood). We also provide the model prediction illustration in Appendix~\ref{Sec:model_pred_compare} and the averaged-context-trajectory-size performance results in Table~\ref{tab:mse_all_model_comparison} and Table~\ref{tab:nll_all_model_comparison} in Appendix~\ref{Sec:add_res}.

    \begin{table}[h]
    \captionsetup{font=scriptsize}
    \caption{Model comparison on \textbf{interpolation} tasks (mean-squared-error $\times 10^{-2}$) for a range of dynamical systems.}
    \centering
    \begin{adjustbox}{max width=\textwidth}
    \begin{tabular}{llccccc}
      \toprule
      Models & Lotka-Volterra ($2d$) & Brusselator ($2d$) & Selkov ($2d$) & SIR ($3d$) & Lotka-Volterra ($3d$) & SIRD ($4d$) \\
      \midrule
          NODEP-$\lambda=0$ & $38.9 \pm 0.52$ & $15.5 \pm 0.379$ & $7.12 \pm 0.56$ & $367.5 \pm 58.0$ & $26.8 \pm 1.53$ & $188.1 \pm 21.8$ \\
      SANODEP-$\lambda=0$ & $\bm{32.1} \pm \bm{0.55}$ & $\phantom{6}\boldsymbol{8.7 \pm 1.11}$ & $\boldsymbol{0.52 \pm 0.03}$ & $\boldsymbol{254.3 \pm 60.3}$ & $\boldsymbol{25.1 \pm 1.85}$ & $\boldsymbol{137.2 \pm 12.8}$ \\

      \bottomrule
    \end{tabular}
    \end{adjustbox}
    \label{tab:mse_interp_comparison_nodep_sanodep}
    \end{table}

\textbf{Does System Awareness Improve Predictive Performance?}  Table.~\ref{tab:mse_interp_comparison_nodep_sanodep} demonstrates that SANODEP is better than NODEP in all problems, indicating the benefit of capturing system information. 


\begin{figure}[h]
    \centering
    \includegraphics[width=1.0\linewidth]{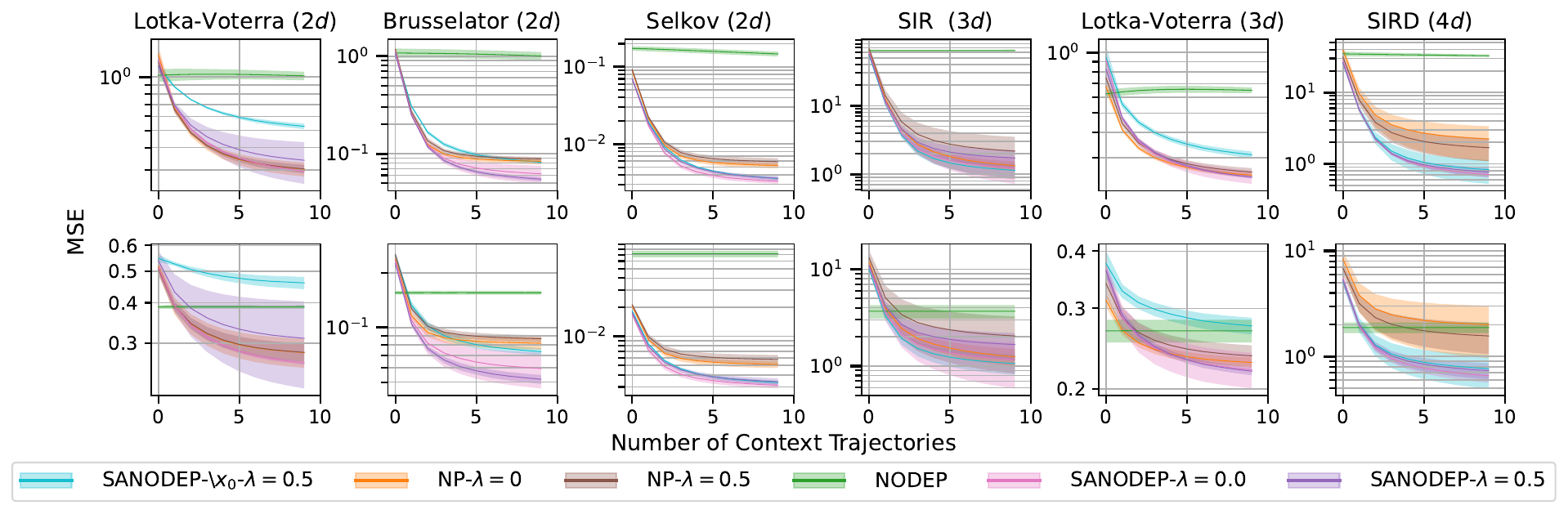}
    \caption{Model evaluation performance comparison on a different number of context trajectories. The first row corresponds to forecasting performance (prediction with only the initial condition known), and the second row represents an interpolating setting. It can be seen that SANODEP is either on-par or marginally better than NP for most problems, and the initial condition augmented encoder-based model variant provides more robust performance across problems. }
    \label{fig:sequential_model_comparison}
\end{figure}

\textbf{Does Incorporating Temporal Information Enhance Model Performance?} Illustrated in Fig.~\ref{fig:sequential_model_comparison}, comparing with NP, except for the Lotka-Voterra problem where the SANODEP trained with mixing probability $\lambda=0.5$ shows larger predictive variance, and Lotka-Voterra ($3d$) where NP demonstrates better performance when context trajectory numbers are small, for the rest of the problems, SANODEP is either on par or noticeably better than NP irrespective with the mixing probability. Indicating that the incorporation of right inductive bias has the potential to help modeling.

\textbf{Does Augmenting Initial Conditions Improve Modeling?} We further evaluate the performance of SANODEP by comparing models with and without initial condition augmentation, as depicted in Fig.~\ref{fig:sequential_model_comparison} (labeled as SANODEP-$\backslash x_0$ for without augmentation). The results intriguingly show that even for first order systems, models augmented with initial conditions exhibit greater robustness across various systems. This demonstrates the empirical benefits of our context aggregation approach.  

\textbf{How Do Meta-Learned Models Compare with Non-Meta-Learned Models?} We provide the sequential model evaluation plot in Fig.~\ref{fig:seq_model_eval_mse_with_gp} including GPs in Appendix~\ref{Sec:add_res}. As expected, the meta-learned models show superior accuracy when the number of trajectories is small (as we have also illustrated in Fig.~\ref{fig:sequential_opt_illu}), aligning with the problem setting of this study. We also remark that GP can quickly overtake when the context trajetcory number is abundant, since the Neural Process typed model is known to suffer from \textit{under fit} when the data is abundant.

\subsection{Few-Shot Bayesian Optimization} \label{sec: meta_bo_exp}
\begin{figure}[h!]
\centering
\includegraphics[width=1.0\textwidth]{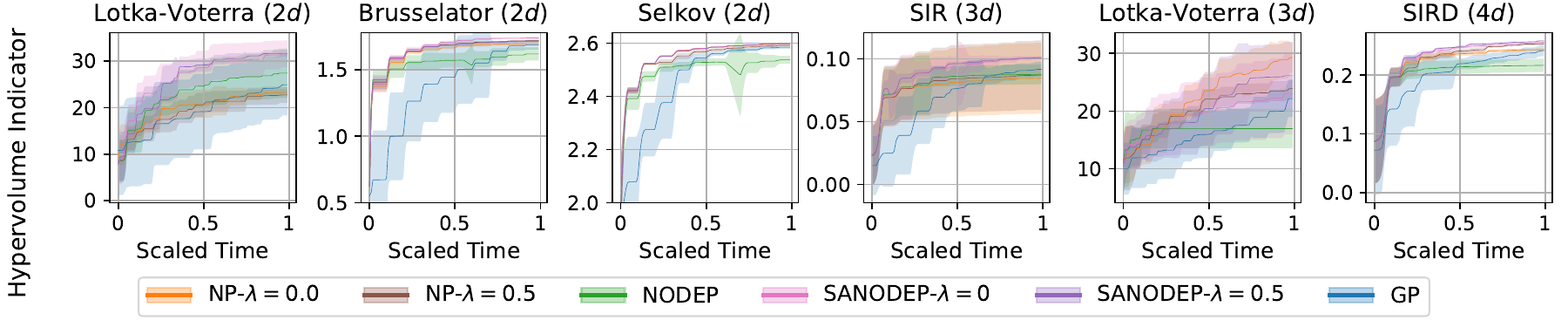}
    \caption{Comparison of few-shot BO performances in terms of mean $\pm$ standard deviation. Except for the Lotka-Voterra ($3d$) problem, SANODEP-based few-shot BO demonstrates competitive performance on all the rest of the problems. }
    \label{fig:meta_bo_comparison}
\end{figure}
\textbf{Baselines}: We conduct few-shot BO using the meta-learned models (SANODEP, NODEP, NP) trained in the previous section using our optimization framework (Algorithm \ref{alg: opt_framework}). For reference, we provide GP-based BO (GP-BO) using the same framework. 

\textbf{Optimization}: We optimize dynamical system realizations sampled from $\mathcal{F}$. The parameter settings, the objective function definitions and performance indicator calculation are provided in Table \ref{tab:optimization_problem_formulations}. For simplicity, in each optimization problem, we assume the maximum number of observations per trajectory is $\Delta t = \frac{t_{max} - t_0}{10}$. We start with one randomly sampled trajectory with uniformly spaced observations and have an additional budget to query 10 trajectories. Each optimization is repeated 10 times with different random seeds.
\newline\textbf{Results} are reported in terms of the scaled experimental time versus the \textit{hypervolume indicator}, with the reference point also provided in Table \ref{tab:optimization_problem_formulations} (we also report the results of PI-SANODEP that will be introduced in next section.). 

\textbf{Does Meta-learning Helps Few-Shot Optimization in Dynamical Systems} We highlight the merits of few-shot BO, as it consistently demonstrates significant convergence speed improvements during the early stages of BO. 

\textbf{Does Dynamical System Informed Model Behaves Better than Standard Model for Optimization} Exept for Lotka-Voterra $(3d)$, SANODEP demonstrate better performance compared with Neural Processes. Again validate the potential of a deliberate consideration of temperal information. 

\textbf{Does Weight Training Objective help Optimization}: For SANODEP, except in the case of the Brusselator, combining training with bi-scenario losses shows either comparable or slightly better performance than training solely in the interpolation setting. This suggests that careful design of the loss function for optimization purposes may offer modest benefits.

Finally, we provide the time profile of different models in the Appendix.~\ref{App: complexity_analysis}, in general, as the involvement of simulation in inference process, SANODEP and NODEP takes more time than NP and GP when perform optimization.

\section{On the strength of the prior information, $\mathcal{F}$}\label{Sec: PriorInfo}

Previously, we assumed that the prior information, the kinetic model, was known to formulate the task distribution. In practice, the level of prior information can vary across problems. In order to better understand how the different levels of prior information can impact the model performance in meta-learning of dynamical systems. We conduct some preliminary investigation on both strengthened and weakened priors in this section. 

\subsection{Strong Prior information: Physics-Informed SANODEP}
When $\boldsymbol{f}$ is a realization of $\mathcal{F}$ and we know the parametric form of the kinetic model, we can benefit from explicitly incorporating the dynamical system within the model, instead of relying on a latent ODE model structure that is agnostic to the dynamical system, we can easily integrate this kinetic form in SANODEP's model structure. We call this model variant Physical Informed (PI)-SANODEP.

As a concrete example, consider $\boldsymbol{f}$ as a realization of the Lotka-Volterra (LV) problem with unknown parameters. Under the SANODEP framework, we can replace $\boldsymbol{f}_{nn}$ in Eq.~\ref{Eq: nodep_solve} with the following physics-informed form:
\begin{equation}
\frac{d}{dt} \begin{bmatrix} 
x_1 \\ 
x_2 
\end{bmatrix} 
= \begin{bmatrix} 
\alpha_{\theta} x_1 - \beta_{\theta} x_1 x_2 \\ 
\delta_{\theta} x_1 x_2 - \gamma_{\theta} x_2 
\end{bmatrix}
\label{Eq: LV_param_form}
\end{equation}
In this form, $\boldsymbol{u}_{sys} = [\alpha_{\theta}, \beta_{\theta}, \delta_{\theta}, \gamma_{\theta}]$, and $\boldsymbol{l}(t_0) = \boldsymbol{x}_0$ and $\boldsymbol{l}(t)=\boldsymbol{x}(t)$\footnote{In PI-SANODEP, since the ODE is explicitly defined as the kinetic form (e.g., Eq. \ref{Eq: LV_param_form}) the decoder is only used to predict variance.}. Since the parameters of the LV system are strictly positive.  We use log-normal distributions for the conditional prior and variational posterior for $\boldsymbol{u}_{sys}$ in the encoder. By additionally incorporating a likelihood regularization term in Eq.~\ref{Eq: ELBO_form} (Appendix.~\ref{App: detailes_of_pi_sanodep}), as we have spotlighted in 
Fig.~\ref{fig:pi_sanodep} and exhaustively demonstrated in Fig.~\ref{fig:pi_sanodep_all_rest}, we found PI-SANODEP can significantly enhance the model performance.


\begin{figure}[h]
    \centering
\includegraphics[width=\textwidth]{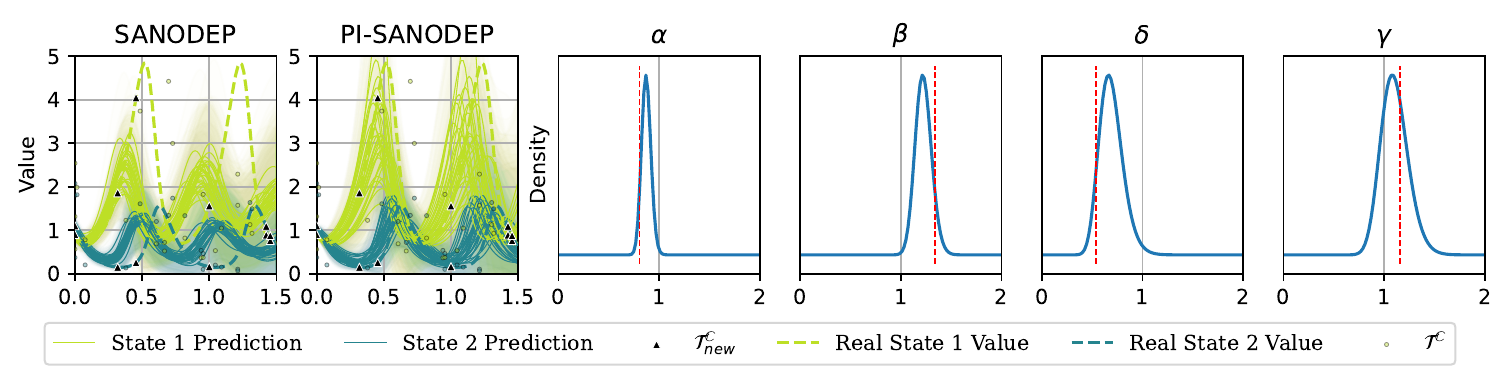}
    \caption{Model Comparison of SANODEP vs Physics Informed (PI) SANODEP, PI-SANODEP provides remarkable better prediction accuracy. In addition, PI-SANODEP provides a reasonable estimation (\textcolor{NavyBlue}{curve}) on system parameters (\textcolor{red}{dashed line}).}
    \label{fig:pi_sanodep}
\end{figure}


\textbf{Meta-learn for Concurrent Prediction and Parameter Estimation}: 
Benefit from the amortized variational inference through the encoder, as we have additionally outlined in Fig.~\ref{fig:pi_sanodep}, we note that PI-SANODEP also demonstrates the potential of concurrent few-shot prediction together with parameter estimation of dynamical systems. Compared with Bayesian inference-based parameter estimation strategies (e.g. \cite{wu2024data}), which may take minutes to perform estimation, PI-SANODEP shares the same advantage as \textit{simulation-based inference} \citep{cranmer2020frontier} and can estimate almost instantly since it only requires a forward pass of the encoder. We report the parameter estimation on the rest of the problems in Fig.~\ref{fig:pi_sanodep_all_rest} in Appendix.~\ref{App: detailes_of_pi_sanodep} and leave an extensive investigation of meta-learning-based parameter estimation as future work. 

\subsection{Weaker Prior of $\mathcal{F}$: Cross-domain generalization possibility} 
What if we know nothing about $\mathcal{F}$? From meta-learning perspective, to still frame the problem as in-distribution evaluation, this demands an extremely flexible prior for $\mathcal{F}$ to be utilizable in the target dynamical system. To our best of knowledge, meta-learning dynamical systems to be leveraged in cross-domain scenarios is still an open direction. Certain approaches have resorted to incorporating additional dynamical system properties  (e.g., energy conserving system \cite{song2023towards}) which limit the scope of system types. We propose and leverage a vector valued GP prior for $\mathcal{F}$ as a flexible task distribution. We show in Appendix.~\ref{app: vec_gp_prior} that it is able generate diverse dynamical systems. We also demonstrate in Fig.~\ref{fig:gp_ode},\ref{fig:cross_domain_generalization}  that SANODEP can also be trained on this task distribution albeit at detriment to model fitting capability.



\section{Discussion}
Bayesian Optimization (BO) of unknown dynamical systems is an underexplored area, with existing approaches often relying on unsuitable surrogate models. Addressing this gap, while minimising the number of costly optimization steps, we have developed a novel few-shot BO framework. Our approach extends the Neural ODE Processes (NODEP) into the System-Aware Neural ODE Processes (SANODEP), which meta-learns the system's prior information to enhance the optimization process. Through extensive benchmark experiments, we demonstrated SANODEP's  potential over NODEP. Our results show that SANODEP, equipped with an optimization-driven loss function, offers competitive performance compared to other non-dynamic-aware meta-learning models, and can offer additional possibilties including parameter estimation , which can't be said with other non-dynamic-aware meta-learning models. 
\newline\textbf{Limitations}: A key limitation of SANODEP, inherited from Neural Processes, is the issue of underfitting. This necessitates a trade-off between few-shot functionality and model fitting capability. Additionally, while strong prior information enhances performance in scenarios with well-defined priors, the model performance is limited by a lack of prior information about the dynamical system. Although our proposed task distribution shows promise for cross-domain generalization, SANODEP still faces a trade-off between model fitting and prior flexibility, indicating that it cannot be used out of the box as a generic probabilistic model to optimize arbitrary dynamical system without knowing its kinetic form.  
\newline\textbf{Future investigations}: Future work will aim to develop more adaptable conditional distributions to better capture a wider range of dynamical systems. This will address the trade-off between model performance and the flexibility needed to handle diverse and complex dynamical systems without relying heavily on strong priors.


\bibliographystyle{plainnat}
\bibliography{references}   
\newpage
\appendix

\part{Appendix} 
\renewcommand{\contentsname}{Appendix Contents} 
\renewcommand\contentsname{whatever}
\parttoc 

\section{Problem Setting and Model Structures} \label{app: pb_illu_nomenclature}
\subsection{Problem Setting Illustrations}
\begin{figure}[h]
    \centering
    \begin{tikzpicture}
        \draw[thick,->] (0,0) -- (6,0) node[anchor=north west] {$t$};
        \draw[thick,->] (0,0) -- (0,4) node[anchor=south east] {$x_0$};
        
        \draw (0.5 cm,1pt) -- (0.5 cm,-3pt) node[anchor=north, font=\scriptsize] {${t_{new}}_{1}^{\mathbb{T}}$};
        \draw (1.5 cm,1pt) -- (1.5 cm,-3pt) node[anchor=north, font=\scriptsize] {${t_{new}}_{2}^{\mathbb{T}}$};
        \draw (2.5 cm,1pt) -- (2.5 cm,-3pt) node[anchor=north, font=\scriptsize] {${t_{new}}_{3}^{\mathbb{T}}$};
        \draw (5.0 cm,1pt) -- (5.0 cm,-3pt) node[anchor=north, font=\scriptsize] {${t_{new}}_{N}^{\mathbb{T}}$};
        \node[align=left] at (3.5 cm, -0.5) {...};
        
        \foreach \x in {0, 1, 2, 3.5, 5} {
            \filldraw[black] (\x, 0.7) circle (1pt);
        }
        
        \foreach \y in {1.1, 1.4, 1.7} {
            \filldraw[black] (2.5, \y) circle (0.5pt);
        }
        
        \foreach \x in {0, 1.1, 2.1, 3.3, 4.8} {
            \filldraw[black] (\x, 2) circle (1pt);
        }
        
        \filldraw[black] (0, 2.7) circle (1pt);
        
        \foreach \x in {0.5, 1.5, 2.5, 5} {
            \node[draw, star, star points=5, star point ratio=2.25, fill=black, inner sep=0.8pt] at (\x, 2.7) {};
        }
        
        \foreach \x in {0, 1.1, 2.1, 3.3, 4.8} {
            \filldraw[black] (\x, 3.5) circle (1pt);
        }
        
        \node[align=left] at (-0.5, 3.5) {$\mathcal{T}_1^\mathbb{C}$};
        \node[align=left] at (-0.5, 2) {$\mathcal{T}_2^\mathbb{C}$};
        \node[align=left] at (-0.5, 2.7) {$\mathcal{T}_{new}^\mathbb{C}$};
        \node[align=left] at (-0.5, 0.7) {$\mathcal{T}_M^\mathbb{C}$};
    \end{tikzpicture}
\caption{Illustration of the initial condition optimization problem within one state variable ($d=1$) ODE system. The optimization involves a forecasting scenario: having observed the context set $\mathbb{C} := \left(\cup_{i=1}^M \mathcal{T}_{i}^{\mathbb{C}} \right) \cup \mathcal{T}_{new}^\mathbb{C}$ (including the \textit{new} trajectory's initial condition, illustrated as \scalebox{0.4}{\textbf{\ding{108}}}), one tries to predict the target state values $\boldsymbol{X}_{new}^{\mathbb{T}} = [{\boldsymbol{x}_{new}}_{1}^{\mathbb{T}}, \ldots, {\boldsymbol{x}_{new}}_{N}^{\mathbb{T}}]$ at target times $\boldsymbol{T}_{new}^\mathbb{T} = [{t_{new}}_1^\mathbb{T}, \ldots, {t_{new}}_N^\mathbb{T}]$ (illustrated as \scalebox{0.75}{\textbf{\ding{72}}}).}
    \label{fig:interpolating_scenario_illustration}
\end{figure}

\subsection{Model Structure}\label{App: model_structure}
The model structure and hyperparameters of SANODEP are summarized as follows:  

\begin{itemize}
    \item \textbf{Initial State Encoder}: $\boldsymbol{r}_{init} = \phi_{init}([t_0, \boldsymbol{x}_0]), \phi_{init}: \mathbb{R}^{d + 1} \rightarrow \mathbb{R}^r$ consists of three dense layers and activation functions in between.  
    \item \textbf{Augmented State Encoder}: $\boldsymbol{r}_i= \phi_r([t_i, \boldsymbol{x}_0, \boldsymbol{x}_i])$, $\phi_r := \mathbb{R}^{2d+1}\rightarrow \mathbb{R}^r $ consists of three dense layers and activation functions in between. 
    \item \textbf{System Context Aggregation}: $\boldsymbol{r}_{sys} = \frac{1}{N}\sum_{i=1}^N \boldsymbol{r}_i$.
    \item \textbf{Context to Hidden Representation }: $\boldsymbol{h}_{sys} = \phi_{sys}(\boldsymbol{r}_{sys})$, : $\boldsymbol{h}_{init} = \phi_{init}(\boldsymbol{r}_{init})$, $\phi_{sys}(\cdot)$ and $\phi_{init}(\cdot)$ are one dense layer followed by a context encoder activation function, $\boldsymbol{h}_{sys} \in \mathbb{R}^h, \boldsymbol{h}_{init} \in \mathbb{R}^h$. 
    \item \textbf{Hidden to Variational posterior of $\boldsymbol{u}_{sys}$}: $q(\boldsymbol{u}_{sys} | \cdot) = \mathcal{N}\left({\phi_{\mu}}_{sys}(\boldsymbol{h}_{sys}), diag\left({\phi_{\sigma}}_{sys}(\boldsymbol{h}_{sys})^2\right) \right)$, ${\phi_{\mu}}_{sys}: \mathbb{R}^h \rightarrow \mathbb{R}^{d_{sys}}$ is a dense layer, ${\phi_{\sigma}}_{sys} :=  \mathbb{R}^h \rightarrow \mathbb{R}^{d_{sys}}$ 
 is defined as $\sigma_{lb} + 0.9 * \text{softplus}(\text{Dense}(\cdot))$ , ${\sigma_{lb}}$ is a hyperparameter.
    \item \textbf{Hidden to Variational posterior of $\boldsymbol{L}_{0}$}: $q({\boldsymbol{L}_0}_{new}^\mathbb{T} | \cdot) = \mathcal{N}\left({\phi_{\mu}}_{init}(\boldsymbol{h}_{init}), {\phi_{\sigma}}_{init}(\boldsymbol{h}_{init})^2\right)$. ${\phi_{\mu}}_{init} := \mathbb{R}^h \rightarrow \mathbb{R}^l$ and ${\phi_{\sigma}}_{init}:\mathbb{R}^h \rightarrow \mathbb{R}^l$ has the same structure as above, but with independent weights.
    \item \textbf{ODE block}: $\phi_{ODE}[\boldsymbol{l}, \boldsymbol{u}_{sys}, t] \rightarrow \boldsymbol{l'}$, $\phi_{ODE} := \mathbb{R}^{d_{sys} + l + 1} \rightarrow \mathbb{R}^l$, $\phi_{ODE}$ is consisted of three dense layers with nonlinear activation functions in between. 
    \item \textbf{Decoder}: $\phi_{dec}[\boldsymbol{l}(t), \boldsymbol{u}_{sys}, t] \rightarrow \boldsymbol{x}(t)$, $\boldsymbol{l} = \mathcal{N}\left({\phi_{\mu}}_{dec}[\boldsymbol{l}(t), \boldsymbol{u}_{sys}, t],\ {\phi_{\sigma}}_{dec}^2[\boldsymbol{l}(t), \boldsymbol{u}_{sys}, t] \right)$, where ${\phi_{\mu}}_{dec}$ and ${\phi_{\sigma}}_{dec}$ has the same structure as variational posterior, but with independent weights.
\end{itemize}

\textbf{Model Hyperparameters}
\begin{itemize}
        \item Encoder $\phi_r$ output dimension $r$: 50
        \item Context encoder ($\phi_r, \phi_{sys}, \phi_{init}$)  activation function: SiLU
        \item Encoder hidden dimension $h$:  50
        \item ODE layer ($\phi_{ODE}$) activation function: Tanh
        \item ODE layer hidden dimension: 50
        \item Decoder hidden dimension: 50
        \item Latent ODE state dimension $l$: 10
        \item Context to latent dynamics (${\phi_{\mu}}_{sys}, {\phi_{\sigma}}_{sys}$) activation function: SiLU
        \item Context to latent initial (${\phi_{\mu}}_{init}, {\phi_{\sigma}}_{init}$) condition activation function: SiLU
        \item Latent dynamics dimension $d_{sys}$: 45
        \item Variational posterior $q({\boldsymbol{L}_0}_{new}^\mathbb{T}| \cdot)$ variance lower bound: $\sigma_{lb} = 0.1$
        \item Variational posterior $q(\boldsymbol{u}_{sys} | \cdot)$ variance lower bound: $\sigma_{lb} = 0.1$
        \item Decoder ($\phi_{dec}, \phi_{{\mu}_{dec}}, \phi_{{\sigma}_{dec}}$) activation function: SiLU
        \item ODE solver: $\code{Dopri5}$ with $\code{rtol}=1e-5$ and $\code{atol}=1e-5$
\end{itemize}

We use SiLU instead of ReLU as in the original NODEP is to enforce the differentiability, which we elaborate on in the following subsection. 
\subsection{On the Differentiability of the Encoder and Decoder} \label{App: differentiability}
Since the optimization takes the initial condition together with the time as decision variables, we need the model output to be differentiable w.r.t.\ these quantities. NODEP overlooked this part and utilized ReLu activation functions originally, we elaborate that the differentiability requirement practically affects our choice of activation functions. 


\textbf{Time derivative} Without loss of generality, we assume we optimize a function $g$ that takes the output of the decoder ${\phi_{\mu}}_{dec}[\boldsymbol{l}(t), \boldsymbol{u}_{sys}, t]$ (defined in Appendix.~\ref{App: model_structure}) as its input \footnote{we note that this is sufficient to show that the requirements of time differentiability irrespective to the consideration of ${\phi_{\sigma}}_{dec}[\boldsymbol{l}(t), \boldsymbol{u}_{sys}, t]$}. The derivative of $g$ w.r.t $t$ is:

\begin{equation}
\begin{aligned}
& \frac{d g\left({\phi_{\mu}}_{dec}[\boldsymbol{l}(t), \boldsymbol{u}_{sys}, t], t\right)}{dt}  = \nabla_{{\phi_{\mu}}_{dec}}g \left(\frac{\partial {\phi_{\mu}}_{dec}}{\partial \boldsymbol{l}(t)} \cdot \boldsymbol{\phi}_{ODEs}[\boldsymbol{l}(t), \boldsymbol{u}_{sys}, t] +\frac{\partial {\phi_{\mu}}_{dec}}{\partial t}\right) + \frac{\partial g}{\partial t}
\end{aligned}
\label{Eq: first_order_derivative}
\end{equation}

\noindent where $\boldsymbol{\phi}_{ODEs}=\cdot$ represents the neural network structure modeling the latent ODEs. It is clear that the differentiability of model output w.r.t time $t$ needs to be enforced through $\frac{\partial {\phi_{\mu}}_{dec}}{\partial t}$, as long as the time $t$ is taken into account in the decoder part.   

\textbf{Initial Condition Derivative}  The derivative of the model output w.r.t the initial condition $\boldsymbol{x}_0$ can be represented as: 

\begin{equation}
\begin{aligned}
 & \frac{d g\left({\phi_{\mu}}_{dec}[\boldsymbol{l}(t), \boldsymbol{u}_{sys}, t], t\right)}{d\boldsymbol{x}(t_0)}  \\ & = \nabla_{{\phi_{\mu}}_{dec}}g \left(\frac{\partial {\phi_{\mu}}_{dec}}{\partial \boldsymbol{l}(t)} \cdot \frac{\partial \boldsymbol{l}(t)}{\partial \boldsymbol{x}(t_0)} + \frac{\partial {\phi_{\mu}}_{dec}}{\partial \boldsymbol{u}_{sys}} \cdot \frac{\partial \boldsymbol{u}_{sys}}{\partial \boldsymbol{x}(t_0)}\right) \\ & = \nabla_{{\phi_{\mu}}_{dec}}g \left[\frac{\partial {\phi_{\mu}}_{dec}}{\partial \boldsymbol{l}(t)} \cdot \left(\frac{\partial \boldsymbol{l}(t_0)}{\partial \boldsymbol{x}(t_0)} + \cdot \right) + \frac{\partial {\phi_{\mu}}_{dec}}{\partial \boldsymbol{u}_{sys}} \cdot \frac{\partial \left({\phi_{\mu}}_{sys}(\boldsymbol{h}_{sys}) + diag\left({\phi_{\sigma}}_{sys}(\boldsymbol{h}_{sys})\epsilon\right)\right)}{\partial \boldsymbol{x}(t_0)} \cdot \right. \\ & \left. \quad ... \cdot \frac{\partial \boldsymbol{h}_{sys}}{\partial \boldsymbol{r}_{sys}}\cdot \frac{\partial \boldsymbol{r}_{sys}}{\partial \boldsymbol{x}_0}\right]
\end{aligned}
\label{Eq: first_order_derivative_init_cond}
\end{equation}

\noindent where $\boldsymbol{r}_{sys}$ and $\boldsymbol{h}_{sys}$ are defined in Appendix.~\ref{App: model_structure}, ${\phi_{\mu}}_{sys}(\boldsymbol{h}_{sys}) + diag\left({\phi_{\sigma}}_{sys}(\boldsymbol{h}_{sys})\epsilon\right)$ represent the reparameterization based sampling of the latent variable $\boldsymbol{u}_{sys}$, where $\epsilon$ is a random sample from Gaussian prior. Note that the derivative of the state value w.r.t to initial condition $\frac{\partial \boldsymbol{l}(t)}{\partial \boldsymbol{x}(t_0)}$ can be obtained through the adjoint method, a derivation has been provided in Appendix C of \cite{chen2018neural}, without a full specification of such cumbersome terms, it is sufficient to indicate the sensibility of having differentiable activation functions in context encoding $\left(\frac{\partial \boldsymbol{r}_{sys}}{\partial \boldsymbol{x}_0}\right)$, context to latent representations $\left(\frac{\partial \boldsymbol{h}_{sys}}{\partial \boldsymbol{r}_{sys}}\right)$, as well as the decoder $\left(\frac{\partial {\phi_{\mu}}_{dec}}{\partial \boldsymbol{l}(t)}\right)$. For a practical calculation of the gradient of the model w.r.t decision variables in both training and optimization, we simply leverage the standard backpropagation through the ODE solver (also known as the "discretize-and-optimize" (e.g. \cite{onken2020discretize, kidger2022neural})) to obtain the gradient due to its speed advantage and minimum implementation effort.

In Figure. \ref{fig:grad_comparison}, we also empirically show the necessity of these differentiable choices by inspecting the time and initial condition derivative with respect to $g(\cdot)$. 

\begin{figure}[h!]
    \centering
    \begin{subfigure}[b]{0.22\textwidth}
        \centering
        \includegraphics[width=\textwidth]{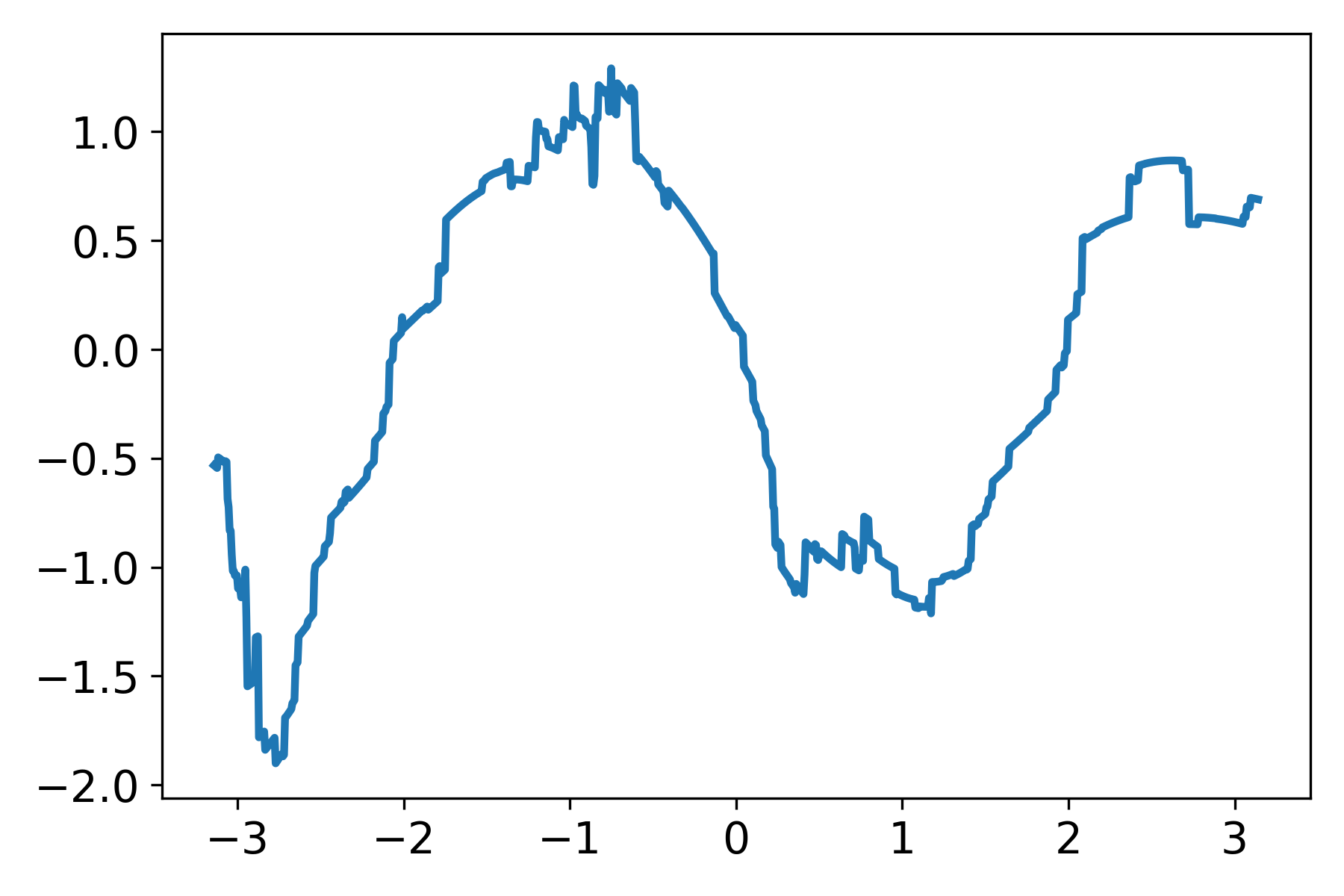}
        \caption{$\frac{\partial x(t)}{\partial t}$ using ReLU.}
    \end{subfigure}
    \hfill
    \begin{subfigure}[b]{0.22\textwidth}
        \centering
        \includegraphics[width=\textwidth]{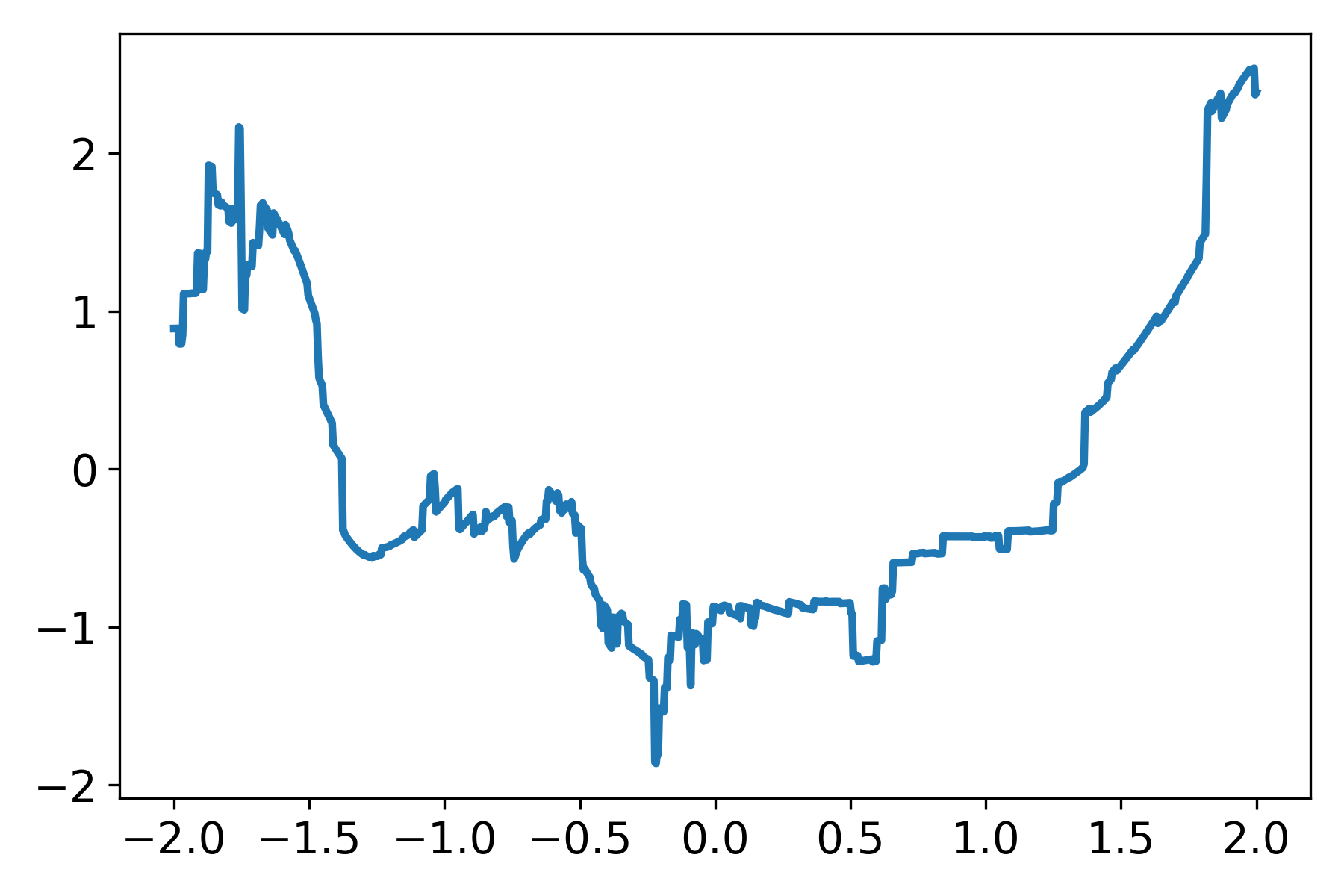}
        \caption{$\frac{\partial x(t)}{\partial x(t_0)}$ using ReLU.}
    \end{subfigure}
    \hfill
    \begin{subfigure}[b]{0.22\textwidth}
        \centering
        \includegraphics[width=\textwidth]{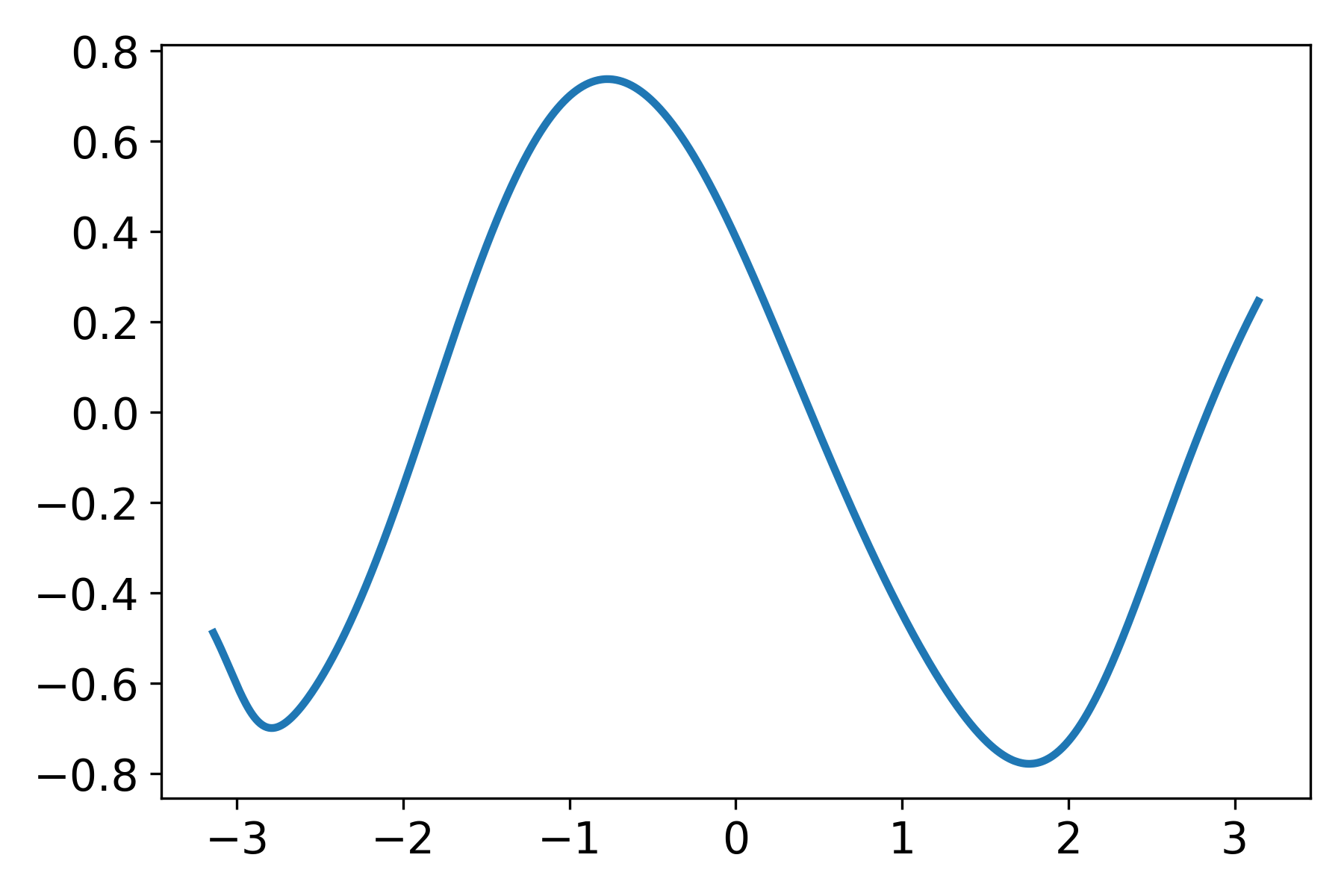}
        \caption{$\frac{\partial x(t)}{\partial t}$ using SiLU.}
    \end{subfigure}
    \hfill
    \begin{subfigure}[b]{0.22\textwidth}
        \centering
        \includegraphics[width=\textwidth]{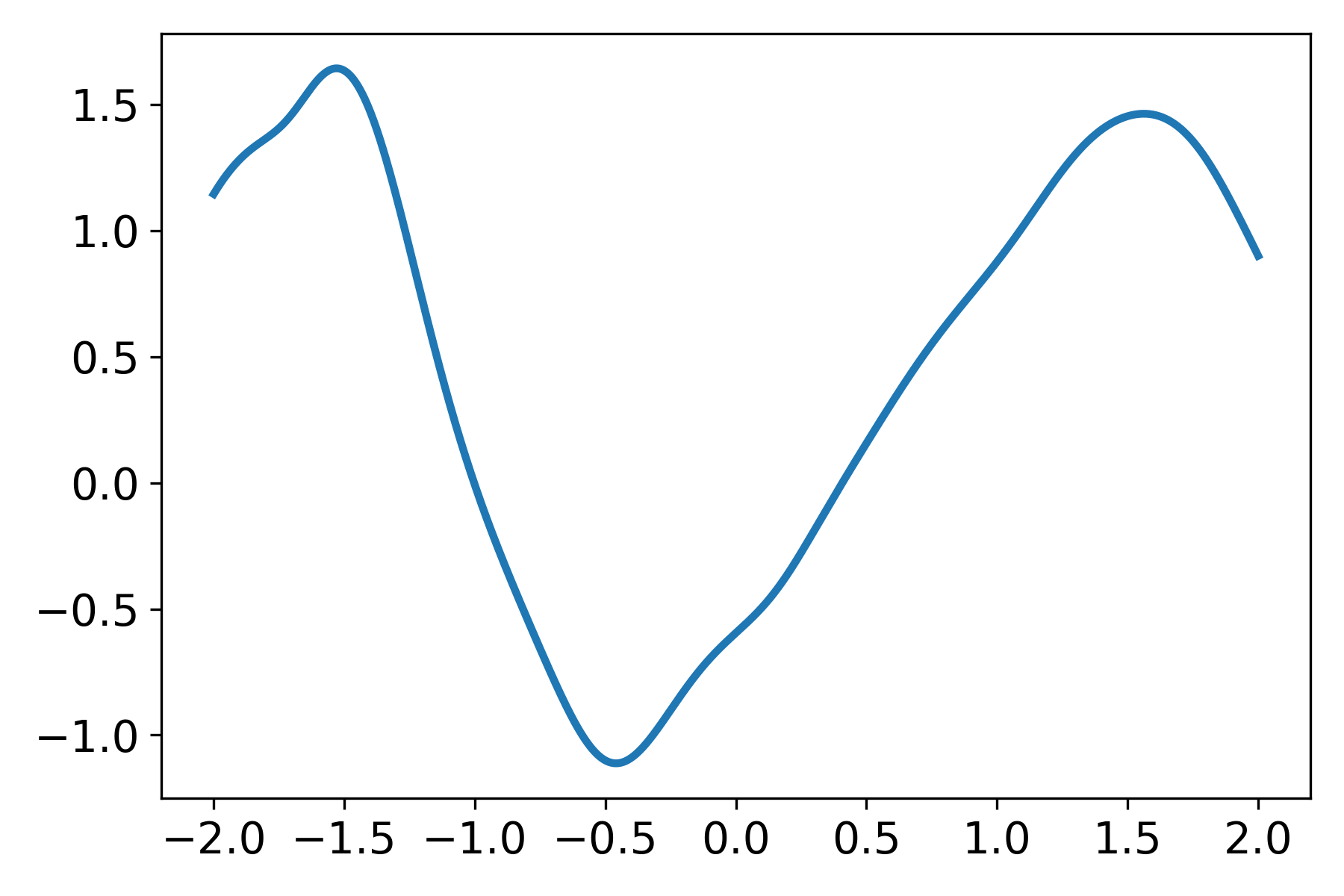}
        \caption{$\frac{\partial x(t)}{\partial x(t_0)}$ using SiLU.}
    \end{subfigure}
    \caption{Initial condition gradient and time derivative w.r.t different activation function choices, the comparison is conducted with the same model structure but different activation functions. }
    \label{fig:grad_comparison}
\end{figure}

\section{ELBO Derivation and Training Details}\label{Sec: loss_derive}
\subsection{SANODEP ELBO derivation}

Note that since the forecasting scenario is a special case of interpolating scenario, below we derive through a unified representation:

\begin{equation}
\begin{aligned}
     & \text{log}\ p\left(\boldsymbol{X}_{new}^{\mathbb{T}} | \mathcal{T}^\mathbb{C}\cup \mathcal{T}_{new}^\mathbb{C}, \boldsymbol{T}_{new}^{\mathbb{T}} \right) \\ & = \text{log}\  \int p\left(\boldsymbol{X}_{new}^{\mathbb{T}} |  \boldsymbol{T}_{new}^{\mathbb{T}}, \boldsymbol{u}_{sys}, \boldsymbol{L}_{0_{new}}^{\mathbb{T}}\right) p\left(\boldsymbol{u}_{sys}| \mathcal{T}^\mathbb{C}\cup \mathcal{T}_{new}^\mathbb{C}\right) p\left(\boldsymbol{L}_{0_{new}}^{\mathbb{T}}\right) d\{\boldsymbol{u}_{sys}, \boldsymbol{L}_{0_{new}}^{\mathbb{T}}\} \\ &  =   \left(\text{log}\int p\left(\boldsymbol{u}_{sys}| \mathcal{T}^\mathbb{C}\cup \mathcal{T}_{new}^\mathbb{C} \right)p\left(\boldsymbol{L}_{0_{new}}^{\mathbb{T}}\right)\frac{q\left( \boldsymbol{u}_{sys}| \mathcal{T}^\mathbb{C}\cup \mathcal{T}_{new}^\mathbb{C}\cup \mathcal{T}_{new}^{\mathbb{T}}\right)}{q\left( \boldsymbol{u}_{sys}| \mathcal{T}^\mathbb{C}\cup \mathcal{T}_{new}^\mathbb{C}\cup \mathcal{T}_{new}^{\mathbb{T}}\right)} \right. \\ & \left. \quad \frac{q\left( \boldsymbol{L}_{0_{new}}^{\mathbb{T}} | \left({t_0^{\mathbb{C}}, \boldsymbol{x_{0_{new}}^{\mathbb{C}}}}\right)\right)}{q\left( \boldsymbol{L}_{0_{new}}^{\mathbb{T}} | \left({t_0^{\mathbb{C}}, \boldsymbol{x_{0_{new}}^{\mathbb{C}}}}\right)\right)}\cdot p\left(\boldsymbol{X}_{new}^{\mathbb{T}} |  \boldsymbol{T}_{new}^{\mathbb{T}}, \boldsymbol{u}_{sys}, \boldsymbol{L}_{0_{new}}^{\mathbb{T}}\right) d\{\boldsymbol{u}_{sys}, \boldsymbol{L}_{0_{new}}^{\mathbb{T}}\}\right) \\ &  \geq   \mathbb{E}_{q\left(\boldsymbol{u}_{sys}| \mathcal{T}^\mathbb{C}\cup \mathcal{T}_{new}^\mathbb{C}\cup \mathcal{T}_{new}^{\mathbb{T}}\right)q\left(\boldsymbol{L}_{0_{new}}^{\mathbb{T}}|  \left({t_0^{\mathbb{C}}, \boldsymbol{x_{0_{new}}^{\mathbb{C}}}}\right)\right)}\text{log}  \left(\frac{p\left(\boldsymbol{u}_{sys}| \mathcal{T}^\mathbb{C}\cup \mathcal{T}_{new}^\mathbb{C}\right)}{q\left(\boldsymbol{u}_{sys}| \mathcal{T}^\mathbb{C}\cup \mathcal{T}_{new}^\mathbb{C}\cup \mathcal{T}_{new}^{\mathbb{T}}\right)}  \right. \\ & \left. \cdot \frac{p\left(\boldsymbol{L}_{0_{new}}^{\mathbb{T}}\right)}{q\left(\boldsymbol{L}_{0_{new}}^{\mathbb{T}}|  \left({t_0^{\mathbb{C}}, \boldsymbol{x_{0_{new}}^{\mathbb{C}}}}\right)\right)} \cdot p\left(\boldsymbol{X}_{new}^{\mathbb{T}} |  \boldsymbol{T}_{new}^{\mathbb{T}}, \boldsymbol{u}_{sys}, \boldsymbol{L}_{0_{new}}^{\mathbb{T}}\right) \right)  \\ & \approx  \mathbb{E}_{q\left(\boldsymbol{u}_{sys}| \mathcal{T}^\mathbb{C}\cup \mathcal{T}_{new}^\mathbb{C} \cup \mathcal{T}_{new}^{\mathbb{T}}\right)q\left(\boldsymbol{L}_{0_{new}}^{\mathbb{T}}|  \left({t_0^{\mathbb{C}}, \boldsymbol{x_{0_{new}}^{\mathbb{C}}}}\right)\right)} \text{log}\  p\left(\boldsymbol{X}_{new}^{\mathbb{T}} |  \boldsymbol{T}_{new}^{\mathbb{T}}, \boldsymbol{u}_{sys}, \boldsymbol{L}_{0_{new}}^{\mathbb{T}}\right) \\ & \quad - \text{KL}\left[q\left(\boldsymbol{u}_{sys}| \mathcal{T}^\mathbb{C}\cup \mathcal{T}_{new}^\mathbb{C} \cup \mathcal{T}_{new}^{\mathbb{T}}\right)| | q\left(\boldsymbol{u}_{sys}| \mathcal{T}^\mathbb{C}\cup \mathcal{T}_{new}^\mathbb{C}\right)\right] -\text{KL}\left[q\left(\boldsymbol{L}_{0_{new}}^{\mathbb{T}} | \left({t_0^{\mathbb{C}}, \boldsymbol{x_{0_{new}}^{\mathbb{C}}}}\right)\right) | | p\left(\boldsymbol{L}_{0_{new}}^{\mathbb{T}}\right)\right]
\end{aligned}
    \label{Eq: ELBO_SANODEP}
\end{equation}

We note the variational posterior $q(\boldsymbol{u}_{sys} | \mathcal{T}^\mathbb{C}\cup \mathcal{T}_{new}^\mathbb{C})$, $q\left( \boldsymbol{L}_{0_{new}}^{\mathbb{T}} | \left({t_0^{\mathbb{C}}, {\boldsymbol{x}_{0_{new}}^{\mathbb{C}}}}\right)\right)$ has been obtained through the encoder in an amortized approach similar in \cite{garnelo2018neural}. 


\begin{algorithm}[h!]


\caption{Learning and Inference in System Aware Neural ODE Processes (SANODEP)}
\label{alg:neural_ode}

\begin{algorithmic}[1]
\REQUIRE ODE system inducing distributions $P$, known trajectory range $[M_{min}, M_{max}]$, batch size of Monte Carlo approximation of initial condition sample $N_{\boldsymbol{x}_0}$, batch size of dynamical systems sample $N_{sys}$, prespecified time 
 grid $\boldsymbol{T}_{grid}:=\texttt{linspace}(t_0, t_{max}, N_{grid})$, minimum and maximum context points within a trajectory $m_{min}, m_{max}$. minimum and maximum extra target points $n_{min}, n_{max}$, initial condition space $\mathcal{X}_0$. 
\STATE Initialize SANODEP model parameters $\theta$ with random seeds.  

\FOR{\code{step} in $\code{training\_steps}$}
    \STATE \textcolor{ForestGreen}{\# \code{Training data generation}}
    \FOR{$j = 1$ to $N_{sys}$}
    \STATE Sample ODE system $\boldsymbol{f} \sim \mathcal{F}$ and  $N_{\boldsymbol{x}_0}$ different initial conditions from $\mathcal{X}_0$, solve $N_{\boldsymbol{x}_0}$ odes at time grids $\boldsymbol{T}_{grid}$ to obtain the dataset $\{\mathcal{T}_1, ..., \mathcal{T}_{N_{\boldsymbol{x}_0}}\}$.
    \STATE Sample known trajectory number $M$ from $\code{Uniform}(M_{min}, M_{max})$.
    \FOR{$l = 1$ to $M$} 
    \STATE \textcolor{ForestGreen}{\# \code{Sample context and target elements within each trajectory}} \\ 
    \STATE Sample $m_l$ from $\code{Uniform}(m_{min}, m_{max})$, sample $n_l$ from $\code{Uniform}(n_{min},n_{max})$.
    \STATE Randomly subsample from $\mathcal{T}_{l}$ to extract the context dataset ${\mathcal{T}_{l}}^{\mathbb{C}}$ and target dataset ${\mathcal{T}_{l}}^{\mathbb{T}}$, where $| {\mathcal{T}_{l}}^{\mathbb{C}} | = m_l$ and $| {\mathcal{T}_{l}}^{\mathbb{T}} | = m_l + n_l$ and ${\mathcal{T}_{l}}^{\mathbb{C}} \subseteq {\mathcal{T}_{l}}^{\mathbb{T}}  $.
    \ENDFOR
    \STATE Concatenate $M$ trajectories context set $\mathcal{T}^{\mathbb{C}} = \{{\mathcal{T}_1^{\mathbb{C}}}, ..., {\mathcal{T}_M^{\mathbb{C}}}\}$ and target set  $\mathcal{T}^{\mathbb{T}} = \{{\mathcal{T}_1^{\mathbb{T}}}, ..., {\mathcal{T}_M^{\mathbb{T}}}\}$.   
    \ENDFOR
    \STATE \textcolor{ForestGreen}{\# \code{Model Prediction}}
    \FOR{$j = 1$ to $N_{sys}$}
    \FOR{$k = 1$ to $N_{\boldsymbol{x}_0 }$}
    \STATE sample $\mathbbm{1}_{forecast}$ from $\code{Bernoulli}(\lambda)$ \ \textcolor{ForestGreen}{\# \code{determine forecasting or interpolating}}
    \IF {$\mathbbm{1}_{forecast} = 1$}  
        \STATE ${\mathcal{T}_{new}^{\mathbb{C}}} = \{({t_0}_{k}, {{\boldsymbol{x}_0}_{k}})\}$ where $({t_0}_{k}, {{\boldsymbol{x}_0}_{k}}) \in \mathcal{T}_k$
    \ELSIF{$\mathbbm{1}_{forecast} = 0$ and $k > M$}
        \STATE Randomly subsample from $\mathcal{T}_{k}$ to extract the context dataset ${\mathcal{T}_{k}}^{\mathbb{C}}$ and target dataset ${\mathcal{T}_{k}}^{\mathbb{T}}$, where $| {\mathcal{T}_{k}}^{\mathbb{C}} | = m$, $| {\mathcal{T}_{k}}^{\mathbb{T}} | = m + n$ and ${\mathcal{T}_{k}}^{\mathbb{C}} \subseteq {\mathcal{T}_{k}}^{\mathbb{T}}  $.
    \ELSE
        \STATE ${\mathcal{T}_{new}^{\mathbb{C}}} = {\mathcal{T}_{k}}^{\mathbb{C}}, {\mathcal{T}}^{\mathbb{C}} = {\mathcal{T}}^{\mathbb{C}} \backslash {\mathcal{T}_{k}}^{\mathbb{C}}$.
    \ENDIF
    \STATE Augmented the context trajectories and target trajectories with the new trajectory $\mathcal{T}^{\mathbb{C}} = \mathcal{T}^{\mathbb{C}} \cup {\mathcal{T}_{new}^{\mathbb{C}}} $, $\mathcal{T}^{\mathbb{T}} = \mathcal{T}^{\mathbb{T}} \cup {\mathcal{T}_{new}^{\mathbb{T}}} $.
    \STATE Compute variational posterior $q(\boldsymbol{u}_{sys} | \mathcal{T}^{\mathbb{T}})$, $q({{\boldsymbol{L}_0}_{new}^\mathbb{T}} | \mathcal{T}^{\mathbb{T}})$, $q(\boldsymbol{u}_{sys} | \mathcal{T}^{\mathbb{C}})$, $q({{\boldsymbol{L}_0}_{new}^\mathbb{T}} | \mathcal{T}^{\mathbb{C}})$ through the encoder block.
    \STATE Sample $\boldsymbol{l}(t_0), \boldsymbol{u}_{sys}$ from $q(\boldsymbol{l}(t_0)|\mathcal{T}^\mathbb{T})$, $q(\boldsymbol{u}_{sys} | \mathcal{T}^\mathbb{T})$. 
    \STATE Solve latent odes as in Eq.~\ref{Eq: nodep_solve} for all times and decode to obtain the model prediction. 
    \STATE Calculate the trajectory wise loss ${\mathcal{L}_{\text{ELBO}}}_k$ based on Eq.~\ref{Eq: ELBO_form}.
    \ENDFOR
    \ENDFOR
    \STATE Average the trajectory wise loss $\mathcal{L}_{\text{ELBO}} = \frac{1}{N_{\boldsymbol{x}_0}}\sum_{o=1}^{N_{\boldsymbol{x}_0}}{\mathcal{L}_{ELBO}}_o$ 
    \STATE Update the model parameter through optimizer $\theta \gets \theta - \eta \nabla_\theta \mathcal{L}_{\text{ELBO}}$
\ENDFOR
\end{algorithmic}
\label{Alg: SANODEP_training}
\end{algorithm}

\subsection{Learning and Prediction Process of SANOEP}\label{App: learn_and_pred}
The practical learning process of SANODEP is summarized in Algorithm.~\ref{Alg: SANODEP_training}, with parameters listed in section.~\ref{Sec:notation} as well. In all of our subsequent experiments, we use $M_{min}=0$, $M_{max}=10$, $N_{\boldsymbol{x}_0} = 100$, $N_{sys}=20$, $N_{grid}=100$, $m_{min}=1$, $m_{max}=10$, $n_{min}=0$, $n_{max}=45$. We note that  all the loops have been practically implemented in batch, and all the subsampling approach has been implemented practically through \texttt{masking} in \texttt{Jax}. 


\section{Time Delay Constraint Optimization}

\subsection{Acquisition Functions} \label{Sec: Acq_Fn}
We advocate that the framework is agnostic to the acquisition function forms as long as it can be written as expected utility form and satisfy the property mentioned in Lemma.~\ref{lm: acq_monotonic}, we however mention the acquisition function form we specifically utilized for performing few-shot BO in Section \ref{sec: meta_bo_exp} as:

\begin{equation}
\begin{aligned}
    & \alpha\left(\boldsymbol{x}_0, t_1, ..., t_N,  p_{\theta}(\boldsymbol{X}^{\mathbb{T}}\vert \mathbb{C}, \boldsymbol{T}^{\mathbb{T}}, \boldsymbol{x}_0)\right) \\ & \approx \frac{1}{N_{MC}}\sum_{i=1}^{N_{MC}}\mathrm{HVI}\left({\phi_{\mu}}_{dec}[\boldsymbol{l}_{i}(\boldsymbol{T}^{\mathbb{T}}), {\boldsymbol{u}_{sys}}_{i}, \boldsymbol{T}^{\mathbb{T}}], \boldsymbol{T}^{\mathbb{T}}, \mathcal{F}^* \right)
\end{aligned}
\end{equation}
In practice, we infer the current Pareto frontier $\mathcal{F}^*$ from SANODEP as well from the $M$ context trajectories $\mathcal{T}^{\mathbb{C}}$. We use $N_{MC}=32$ as the MC sample size to approximate the expected hypervolume improvement in all our experiments.

\textbf{Acquisition Function Optimization} 
We note the acquisition optimizer of both initial condition (Eq.~\ref{Eq: joint_acq_opt_form}), as well as optimal time scheduling  (Eq.~\ref{Eq: last_acq_opt_form}) involves constraint optimization. We utilize trust region-based constraint optimization available in \texttt{Scipy}\footnote{\url{https://docs.scipy.org/doc/scipy/reference/optimize.minimize-trustconstr.html}} to optimize acquisition function in both optimization problems starting with a uniform time scheduling (i.e., \code{linspace($t_0$, $t_1$, $N$)}) . Note that for the initial stage optimization, we additionally leverage a (10 instances) multi-start procedure on the initial condition decision variable only to boost the optimization performance. 

\subsection{Search Space Reduction}
\textbf{Proof of the search space reduction}
We prove that the search space can be reduced, without eliminating the global maximum in acquisition function optimization processes (Eq.~\ref{Eq: joint_acq_opt_form}). 
\begin{lemma}
    For batch acquisition function $\alpha$ defined as expected utility $\alpha = \mathbb{E}_{p(\boldsymbol{x})}(u(\boldsymbol{x}))$ where the utility function $u()$ is monotonic w.r.t set inclusion ($S \subseteq T \Rightarrow u(T) \geq u(S)$), the aquisition function is also monotonic w.r.t set inclusion: $\alpha(T) \geq \alpha(S)$. 
\label{lm: acq_monotonic}
\end{lemma}

\begin{proof}
    Let $T = S \cup T / S $, then $\alpha(T) - \alpha(S) = \mathbb{E}_{p(S)}\mathbb{E}_{p(T / S \vert S)}(u(T)) - \mathbb{E}_{p(S)}\mathbb{E}_{p(T / S \vert S)}(u(S)) = \mathbb{E}_{p(S)}\mathbb{E}_{p(T / S \vert S)}\left(u(T) - u(S) \right) \geq 0 $, hence prove complete. 
\end{proof}

\begin{corollary}
Assume that the utility function \( u(\cdot) \) is monotonic with respect to the inclusion of the set in each set separately: Let \( \boldsymbol{X} = \{X_1, X_2, \ldots, X_n\} \) and \( \boldsymbol{Y} = \{Y_1, Y_2, \ldots, Y_n\} \) be two collections of sets, if \( X_i \subseteq Y_i \) $\forall i \in \{1..., n\}$ then $u(\boldsymbol{X}) \leq u(\boldsymbol{Y})$. Then, the batch acquisition function defined as \( \alpha = \mathbb{E}_{p(\boldsymbol{X})}(u(\boldsymbol{X})) \) satisfies \( \alpha(\boldsymbol{Y}) \geq \alpha(\boldsymbol{X}) \).
\label{coro: batch}
\end{corollary}
\begin{lemma}
    \textup{\texttt{qEHVI}} is a monotonic acquisition function with respect to set inclusion
\label{lm： qEHVI_monotonicity}
\end{lemma}
\begin{proof}
Using Corollary.~\ref{coro: batch}, Lemma \ref{lm： qEHVI_monotonicity} and the well-acknowledged fact that the hypervolume indicator is monotonic with respect to set inclusion on each output dimensionality, the proof is straightforward. We note that the above property is asymptotically hold when \texttt{qEHVI} is monte carlo approximated. 
\end{proof} 

Finally, we start the proof of Theorem. \ref{thm: search_space_reduction}: 
\begin{proof}
    when $N \in \left[1, \lceil\frac{N_{max}}{2}\rceil\right]$, $\exists i \in \left[1, \lceil \frac{N_{max}}{2} \rceil\right]\ s.t.\ t_i - t_{i-1} \geq 2\Delta t$ meaning that one can insert a new batch point in between, forming an augmented set $T$ as input, with Lemma \ref{lm： qEHVI_monotonicity} and Corollary.~\ref{coro: batch}, the proof hence complete. 
\end{proof}
We finally remark that, with Lemma. \ref{lm: acq_monotonic}, a large series of batch improvement-based acquisition functions (e.g., parallel Expected Improvement, parallel Probability of Improvement, and parallel information-theoretic acquisition functions that can be written as expected utility form and also preserve the set inclusion monotonicity property (e.g., \cite{qing2023pf})) also embrace the same search space reduction benefit.

 \subsection{Optimization Framework}\label{Sec: opt_framework}
\begin{algorithm}[h!]
   \caption{Model Assisted Ordinary Differential Equation Optimization Framework}
   \label{alg: adj_grad_acq_opt}
\begin{algorithmic}[1]
   \STATE {\bfseries Input:} maximum number of experiment trajectories $N_{exp}$, minimum spacing $\Delta t$, maximum observations per trajectory $N_{max}$, initial evaluated context datasets $\mathbb{C}$, acquisition function $\alpha$, evaluated trajectory  number $n=1$, initial time lower bound $t_{lb} = t_0$, time range $\tau=[t_{lb}, t_{max}]$, observer $\boldsymbol{f}_{evolve}$  for measurement.
   \WHILE{$n \leq N_{exp}$}
   \STATE $\mathcal{T} = \phi$
   \STATE Construct the probabilistic model for $p(\boldsymbol{x}| t, \mathbb{C}, \boldsymbol{x}_0)$
   \STATE \textcolor{ForestGreen}{\# \code{Initial Condition Optimization}}
   \STATE $\boldsymbol{x}_0^*, N^* =$ arg $\max_{\boldsymbol{x}_0 \in \mathcal{X}, \{t_1=t_0, t_2, ..., t_N \in \tau\}, N \in [\lceil\frac{N_{max}}{2}\rceil, N_{max}]}\ \alpha\left(\boldsymbol{x}_0, t_1, ..., t_N \right)$ \\
   \quad \quad\quad\quad $\text{ s. t.\  } \forall i \in \{1, ..., N\}: t_{i}-t_{i-1} \geq \Delta t$\\
   $t_{lb} = t_0 + \Delta t, \tau:=[t_{lb}, t_{max}], \mathcal{T} = \mathcal{T}\cup\{t_0, \boldsymbol{x}_0^*\}, \mathbb{C}=\mathbb{C} \cup \mathcal{T}$
   \STATE $N_{traj} = N_{max} - 1$ 
   \STATE \textcolor{ForestGreen}{\# \code{Within trajectory Next Measurement Time Scheduling }}
   \WHILE{$N_{traj} >= 1$}  
         \STATE $N^*, t_1^* =$ arg $\max_{\{t_1, t_2, ..., t_N \in \tau\}, N \in [\lceil\frac{N_{traj}}{2}\rceil, N_{traj}]}\ \alpha\left(t_1, ..., t_N \right)$\\
    $\text{ s. t.\  } \forall i \in \{1, ..., N\}: t_{i}-t_{i-1} \geq \Delta t$
      \STATE $\boldsymbol{x}_1^* = \boldsymbol{f}_{evolve}(\boldsymbol{x}_0^*, t_1^*)$ 
      \STATE $t_{lb} = t_1^* +\Delta t$, $\tau := [t_{lb}, t_{ub}]$, $\mathbb{C}=(\mathbb{C} \backslash \mathcal{T}) \cup (\mathcal{T} \cup \{t_1, \boldsymbol{x}_1^*\}), \mathcal{T} = \mathcal{T}\cup\{t_1, \boldsymbol{x}_1^*\}$
      \STATE $N_{traj} = \lfloor \frac{t_{max} - t_{lb}}{\Delta t} \rfloor$
   \ENDWHILE
   \STATE $n = n + 1$
   \ENDWHILE
\end{algorithmic}
\label{alg: opt_framework}
\end{algorithm} 

\section{Experimental Details and Additional Results}

\subsection{Meta Training Data Definition} \label{sec: training_data_describe}

\textbf{Lotka-Voterra}  

\textbf{2D Cases}

\begin{equation}
    \frac{d\boldsymbol{x}}{dt} = \begin{bmatrix}
        \alpha x_1 - \beta x_1 x_2 \\
        \delta x_1 x_2 - \gamma x_2
    \end{bmatrix}
\label{Eq: LV_form}
\end{equation}

The initial condition is sample from $\boldsymbol{x}(t_0) \sim U(0.1, 3)^2$, the dynamics $\mathcal{F}$ are sampled from $\alpha \sim U(\frac{1}{3}, 1)$, $\beta \sim U(1, 2)$, $\delta \sim U(0.5, 1.5)$, $\gamma \sim U(0.5, 1.5)$, $t_{max}=15$. We run the meta-learn on this system distribution with 300 epochs.

\textbf{3D cases} 
\begin{equation}
\frac{d\mathbf{x}}{dt} = \begin{bmatrix}
\alpha x_1  - \beta x_1 x_2 - \epsilon x_1 x_3 \\
\delta x_2 x_1 - \gamma x_2  - \zeta x_2 x_3 \\
\eta x_3 x_2 - \theta x_3 
\end{bmatrix}
\end{equation}
The initial condition is sampled from $\boldsymbol{x}(t_0) \sim U(0.1, 3)^3$. The dynamics $\mathcal{F}$ are sampled from $\alpha \sim U\left(\frac{1}{3}, 1\right)$, $\beta \sim U(1, 2)$, $\delta \sim U(0.5, 1.5)$, $\gamma \sim U(0.5, 1.5)$, $\epsilon \sim U(0.5, 1.5)$, $\zeta \sim U(0.5, 1.5)$, $\eta \sim U(0.5, 1.5)$, and $\theta \sim U(0.5, 1.5)$, with $t_{max} = 15$. We run the meta-learning algorithm on this system distribution for 300 epochs.

\textbf{Brusselator} \cite{prigogine1968symmetry}
\begin{equation}
    \frac{d\boldsymbol{x}}{dt} = \begin{bmatrix}
        A + x_1^2 x_2 - (B+1)x_1 \\
        Bx_1 - x_1^2 x_2
    \end{bmatrix}
\end{equation}
we leverage $A \sim U(0, 1)$, $B \sim U(0.1, 3)$, $\boldsymbol{x}_0 \sim U(0.1, 2.0)^2$, $t_{max} = 15$. We run the meta-learning algorithm on this system distribution for 300 epochs.

\textbf{SIR model}

\begin{equation}
    \frac{dS}{dt} = -\beta S I,\ \frac{dI}{dt} = \beta S I - \gamma I,\ \frac{dR}{dt} = \gamma I
\end{equation}

The initial condition is sampled from $S(t_0) \sim U(1.0, 3.0)$, $I(t_0) = 0.01$, $R(t_0) = 0$, and $\beta \sim U(0.1, 2)$, $\gamma \sim U(0.1, 10)$. $t_{max}=1$. We run the meta-learning algorithm on this system distribution for 300 epochs.

\textbf{SIRD model} 
\begin{equation}
    \begin{aligned}
\frac{dS}{dt} = -\beta S I,\ 
\frac{dI}{dt} = \beta S I - \gamma I - \mu I,\ 
\frac{dR}{dt} = \gamma I,\ 
\frac{dD}{dt} = \mu I
\end{aligned}
\end{equation}
\noindent where $S(t_0) \sim U(10, 30)$, $I(t_0) =0.01$, $R(t_0)=D(t_0)=0$. $t_{max}=1$, $\beta \in [0.5, 2.0]$, $\gamma \in [0.1, 10]$, and $\mu \in [0.1, 5.0]$. We run the meta-learn on this system distribution with 300 epochs. 

\textbf{GP Vector Field}
\begin{equation}
    \frac{d \boldsymbol{x}}{dt} = \boldsymbol{f}(\boldsymbol{x})
\end{equation}
We assume a vector-valued GP as the non-parametric prior for the vector field $\boldsymbol{f} \sim \mathcal{GP}(\boldsymbol{0}, K(\boldsymbol{x}, \boldsymbol{x}'))$, with no correlations between each output as the most generic cases. To sample the vector field, we leverage the parametric approximation of GPs through the random Fourier feature approximation of the RBF kernel \cite{rahimi2007random} as a Bayesian linear model, which is a common approach (e.g. \cite{qing2022spectral}) to obtain differentiable GP samples, the vector fields in practice are generated from kernel lengthscale 0.8 and signal variance 1. The implementation also utilizes the parametric sampling approach of the \texttt{GPJax} \citep{Pinder2022} library.

\subsection{Additional Experimental Results}\label{Sec:add_res}
We report all model's performance on interpolating tasks and forecasting tasks, averaged over known trajectory ranges $M$, in Table.~\ref{tab:mse_all_model_comparison} and Table.~\ref{tab:nll_all_model_comparison}.

    \begin{table}[h!]
    \captionsetup{font=scriptsize}
    \caption{Model comparison (mean-squared-error $\times 10^{-2}$) for a range of dynamical systems.}
    \centering
    \begin{adjustbox}{max width=\textwidth}
    \begin{tabular}{lcccccc}
      \toprule
      Models & Lotka-Volterra (2d) & Brusselator (2d) & Selkov (2d) & SIR (3d) & Lotka-Volterra (3d) & SIRD (4d) \\
      \midrule
    
        \multicolumn{7}{c}{\textbf{Forecasting}} \\
        \midrule
              GP & $87.2 \pm 31.9$ & $39.3 \pm 61.0$ & $31.8 \pm 6.72$ & $1769.93 \pm 980.6$ & $41.4 \pm 15.8$ & $1322.845 \pm 489.4$ \\
      NODEP-$\lambda=0$ & $102.9 \pm 6.85$ & $105.2 \pm 10.8$ & $16.1 \pm 1.14$ & $6367.784 \pm 161.8$ & $64.8 \pm 2.56$ & $3368.685 \pm 189.5$ \\
      NP-$\lambda=0.0$ & $48.6 \pm 3.1$ & $21.6 \pm 0.56$ & $1.6 \pm 0.0916$ & $952.3 \pm 73.6$ & $\boldsymbol{32.8 \pm 0.85}$ & $724.2 \pm 168.8$ \\
      NP-$\lambda=0.5$ & $\boldsymbol{47.3 \pm 1.34}$ & $22.4 \pm 1.47$ & $1.7 \pm 0.0882$ & $1057.981 \pm 207.2$ & $34.8 \pm 1.53$ & $573.0 \pm 88.3$ \\
      SANODEP-$\lambda=0$ & $48.4 \pm 2.09$ & $20.1 \pm 0.816$ & $\boldsymbol{1.22 \pm 0.0376}$ & $\boldsymbol{861.6 \pm 89.0}$ & $36.2 \pm 1.17$ & $426.4 \pm 44.0$ \\
      SANODEP-$\lambda=0.5$ & $50.6 \pm 6.4$ & $\boldsymbol{18.8 \pm 0.467}$ & $1.28 \pm 0.0459$ & $873.1 \pm 82.4$ & $35.5 \pm 0.824$ & $\boldsymbol{410.5 \pm 27.3}$ \\
      \midrule
        \multicolumn{7}{c}{\textbf{Interpolating}} \\
        \midrule
              GP & $52.2 \pm 20.2$ & $9.62 \pm 19.0$ & $11.7 \pm 2.76$ & $\boldsymbol{173.7 \pm 296.8}$ & $\boldsymbol{16.8 \pm 13.7}$ & $280.7 \pm 137.3$ \\
      NODEP-$\lambda=0$ & $38.9 \pm 0.517$ & $15.6 \pm 0.379$ & $7.12 \pm 0.559$ & $367.5 \pm 58.0$ & $26.8 \pm 1.53$ & $188.1 \pm 21.8$ \\
      NP-$\lambda=0.0$ & $33.2 \pm 2.59$ & $10.4 \pm 0.523$ & $0.742 \pm 0.0436$ & $286.2 \pm 29.7$ & $24.7 \pm 0.557$ & $302.0 \pm 110.9$ \\
      NP-$\lambda=0.5$ & $33.2 \pm 1.48$ & $11.0 \pm 0.477$ & $0.807 \pm 0.0638$ & $378.3 \pm 143.0$ & $25.9 \pm 1.39$ & $244.6 \pm 54.3$ \\
      SANODEP-$\lambda=0$ & $\boldsymbol{32.1 \pm 0.556}$ & $8.74 \pm 1.09$ & $\boldsymbol{0.53 \pm 0.03}$ & $257.5 \pm 60.3$ & $25.2 \pm 1.84$ & $\boldsymbol{139.5 \pm 13.5}$ \\
      SANODEP-$\lambda=0.5$ & $36.6 \pm 7.73$ & $\boldsymbol{8.06 \pm 0.265}$ & $0.573 \pm 0.0247$ & $307.8 \pm 53.5$ & $25.2 \pm 0.464$ & $143.8 \pm 14.0$ \\

      \bottomrule
    \end{tabular}
    \end{adjustbox}
    \label{tab:mse_all_model_comparison}
    \end{table}

    \begin{table}[h!]
    \captionsetup{font=scriptsize}
    \caption{Model comparison (negative log-likelihood $\times 10^{2}$) for a range of dynamical systems.}
    \centering
    \begin{adjustbox}{max width=\textwidth}
    \begin{tabular}{lcccccc}
      \toprule
      Models & Lotka-Volterra (2d) & Brusselator (2d) & Selkov (2d) & SIR (3d) & Lotka-Volterra (3d) & SIRD (4d) \\
      \midrule
    
        \multicolumn{7}{c}{\textbf{Forecasting}} \\
        \midrule
              GP & $18.2 \pm 7.71$ & $30.2 \pm 25.7$ & $17.1 \pm 4.95$ & $11143.352 \pm 31169.575$ & $\boldsymbol{19.3 \pm 5.75}$ & $27321.033 \pm 57173.003$ \\
      NODEP-$\lambda=0$ & $119.3 \pm 21.7$ & $79.1 \pm 13.9$ & $136.4 \pm 32.1$ & $1086.616 \pm 331.3$ & $625.5 \pm 174.2$ & $447.7 \pm 86.0$ \\
      NP-$\lambda=0.0$ & $98.9 \pm 52.2$ & $\boldsymbol{17.7 \pm 4.41}$ & $\boldsymbol{4.28 \pm 1.34}$ & $\boldsymbol{162.6 \pm 78.6}$ & $491.2 \pm 122.1$ & $197.5 \pm 151.7$ \\
      NP-$\lambda=0.5$ & $200.9 \pm 88.0$ & $37.7 \pm 18.5$ & $5.93 \pm 0.836$ & $279.7 \pm 198.2$ & $961.1 \pm 454.1$ & $75.7 \pm 42.8$ \\
      SANODEP-$\lambda=0$ & $\boldsymbol{15.5 \pm 6.9}$ & $22.6 \pm 2.98$ & $7.98 \pm 1.39$ & $206.9 \pm 79.4$ & $277.5 \pm 77.9$ & $48.7 \pm 17.4$ \\
      SANODEP-$\lambda=0.5$ & $27.0 \pm 13.0$ & $23.7 \pm 4.55$ & $8.55 \pm 1.42$ & $219.4 \pm 139.7$ & $329.0 \pm 103.8$ & $\boldsymbol{48.3 \pm 11.0}$ \\
        \midrule
        \multicolumn{7}{c}{\textbf{Interpolating}} \\
        \midrule
              GP & $3.27 \pm 1.24$ & $1.88 \pm 2.0$ & $5.27 \pm 1.87$ & $718.0 \pm 2426.002$ & $\boldsymbol{2.69 \pm 1.61}$ & $1045.453 \pm 3397.646$ \\
      NODEP-$\lambda=0$ & $32.6 \pm 6.73$ & $12.7 \pm 2.89$ & $13.9 \pm 2.94$ & $66.4 \pm 18.2$ & $67.1 \pm 22.9$ & $22.9 \pm 8.44$ \\
      NP-$\lambda=0.0$ & $18.7 \pm 7.46$ & $\boldsymbol{0.254 \pm 0.61}$ & $\boldsymbol{-2.29 \pm 0.215}$ & $\boldsymbol{22.4 \pm 10.1}$ & $65.8 \pm 11.2$ & $30.1 \pm 23.1$ \\
      NP-$\lambda=0.5$ & $34.8 \pm 14.0$ & $2.88 \pm 2.01$ & $-1.97 \pm 0.147$ & $40.1 \pm 28.4$ & $122.1 \pm 49.6$ & $12.5 \pm 6.12$ \\
      SANODEP-$\lambda=0$ & $\boldsymbol{2.88 \pm 1.59}$ & $1.24 \pm 0.568$ & $-1.74 \pm 0.308$ & $28.1 \pm 10.5$ & $41.3 \pm 17.7$ & $\boldsymbol{7.17 \pm 2.39}$ \\
      SANODEP-$\lambda=0.5$ & $6.08 \pm 3.63$ & $1.13 \pm 0.596$ & $-1.5 \pm 0.284$ & $31.3 \pm 19.3$ & $47.3 \pm 14.6$ & $7.4 \pm 1.97$ \\

      \bottomrule
    \end{tabular}
    \end{adjustbox}
    \label{tab:nll_all_model_comparison}
    \end{table}

\begin{figure}[h!]
    \centering
    \includegraphics[width=1.0\linewidth]{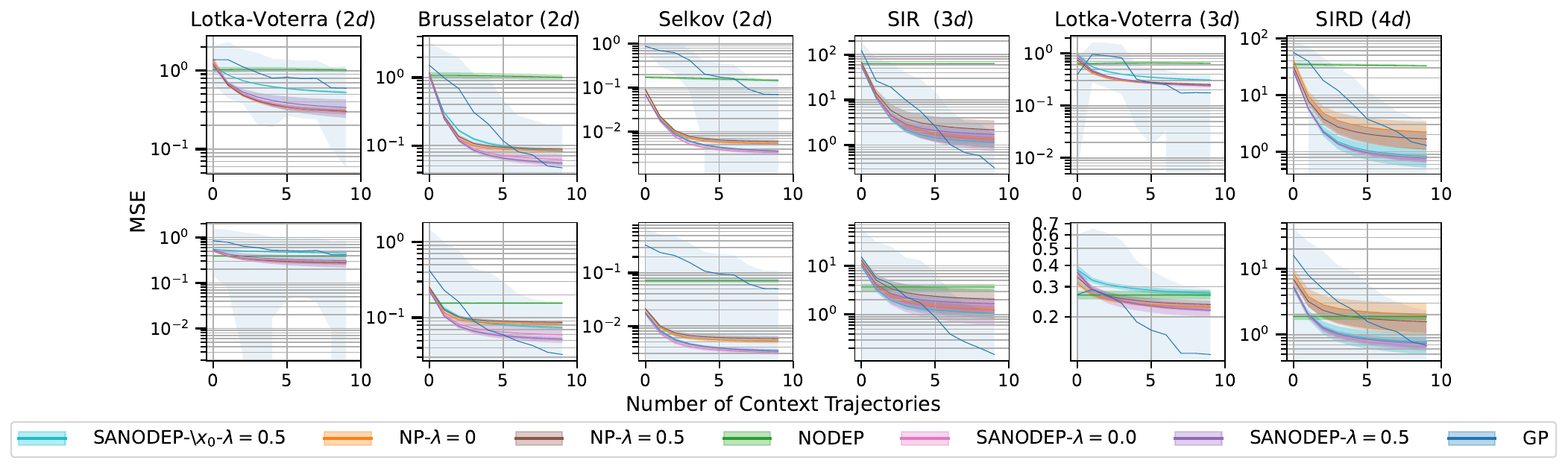}
    \caption{Sequential model evaluation performance including GPs. The first row corresponds to forecasting performance (prediction with only the initial condition known), and the second row represents an interpolating setting. }
    \label{fig:seq_model_eval_mse_with_gp}
\end{figure}

\begin{figure}[h!]
    \centering
    \includegraphics[width=1.0\linewidth]{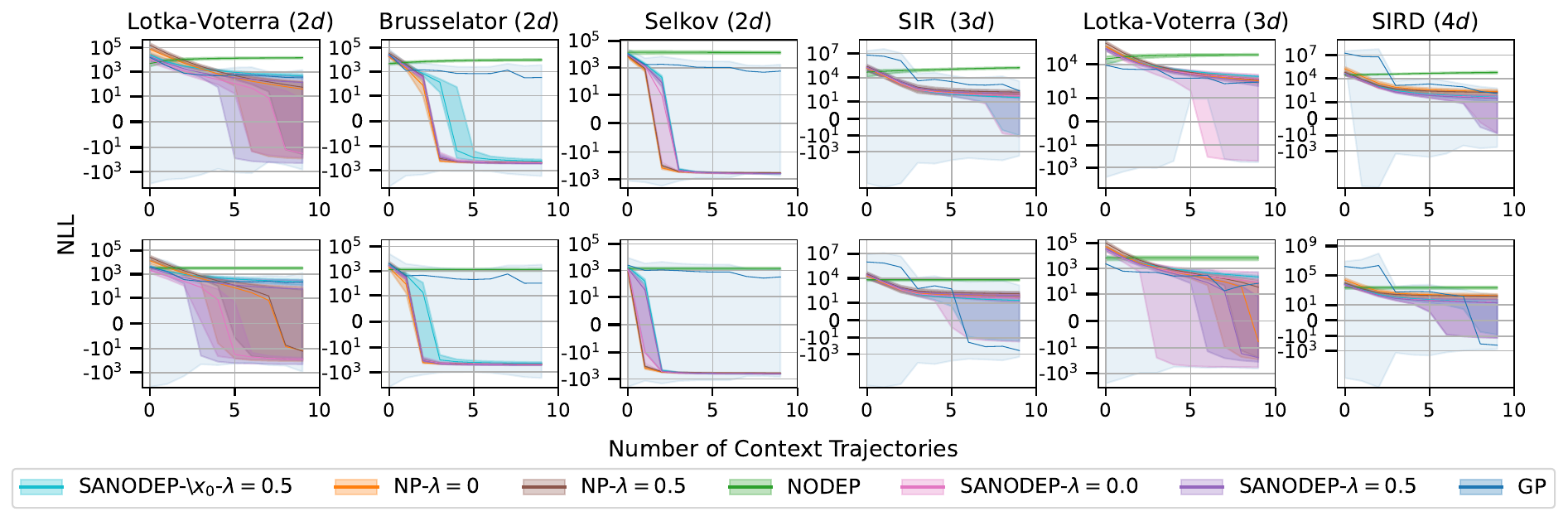}
    \caption{Sequential model evaluation performance in terms of Negative-Log-Likelihood (NLL) including GPs. The first row corresponds to forecasting performance (prediction with only the initial condition known), and the second row represents an interpolating setting. }
    \label{fig:seq_model_eval_nll_with_gp}
\end{figure}

\subsection{Model Prediction Visual Comparison} \label{Sec:model_pred_compare}

We provide the prediction comparison of different models on a specific realization of ODE systems in Fig. \ref{fig:visual_prediction_comparison}, \ref{fig:nodep}, \ref{fig:sanodep_visual}, \ref{fig:pi_sanodep_visual}, \ref{fig:gp}.

\begin{figure}[h!]
    \centering
    \includegraphics[width=\textwidth]{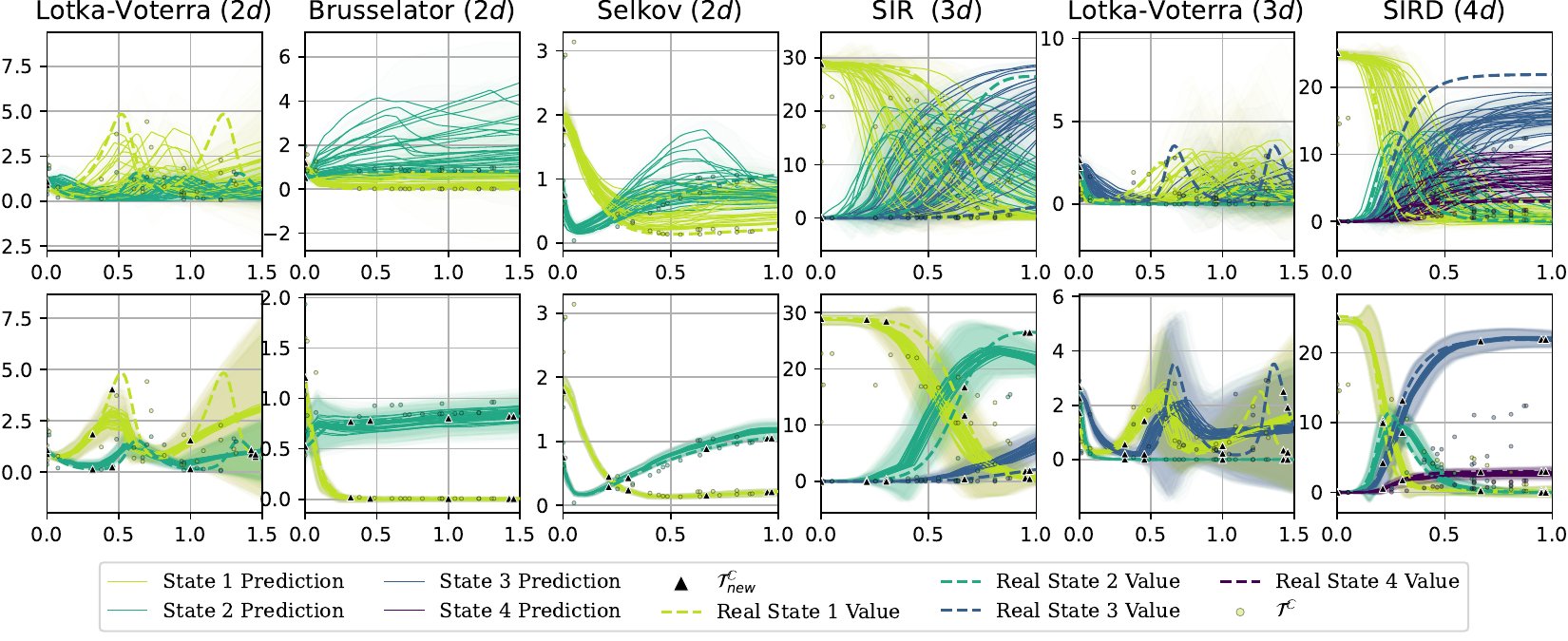}
    \caption{Neural Processes ($\lambda = 0.5$) podel merformance on test system in different meta-learning ODE problems.}
    \label{fig:visual_prediction_comparison}
\end{figure}

\begin{figure}[h!]
    \centering
    \includegraphics[width=\textwidth]{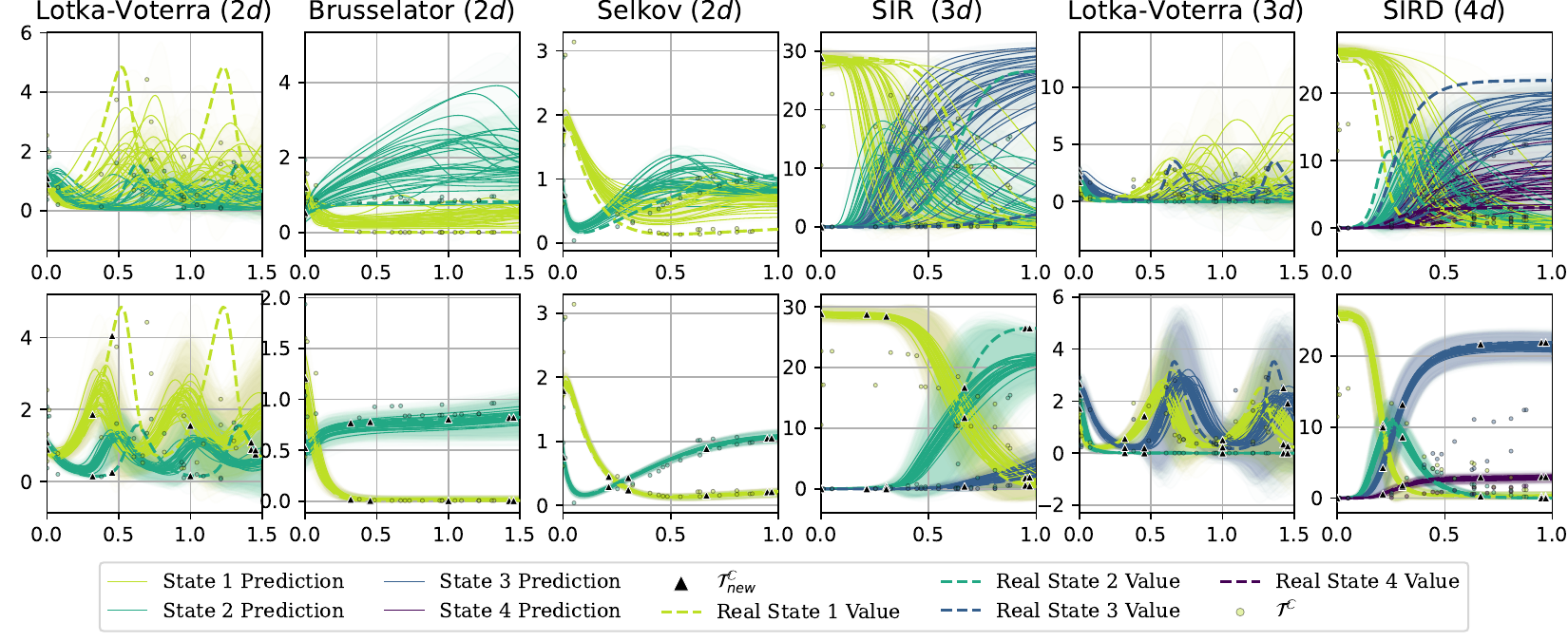}
    \caption{SANODEP-($\lambda = 0.5$) model performance on test system in different meta-learning ODE problems.}
    \label{fig:sanodep_visual}
\end{figure}

\begin{figure}[h!]
    \centering
    \includegraphics[width=\textwidth]{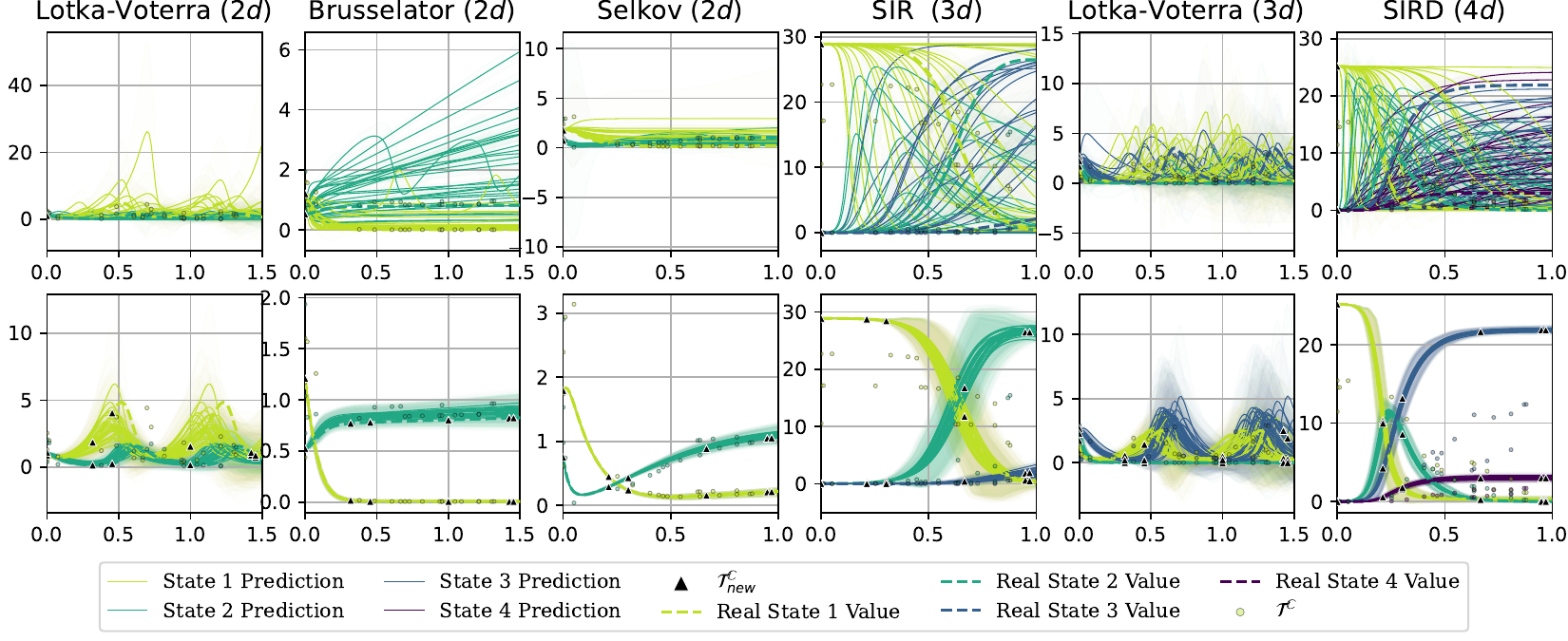}
    \caption{PI-SANODEP-($\lambda = 0.5$) model performance on test system in different meta-learning ODE problems.}
    \label{fig:pi_sanodep_visual}
\end{figure}

\begin{figure}[h!]
    \centering
    \includegraphics[width=\textwidth]{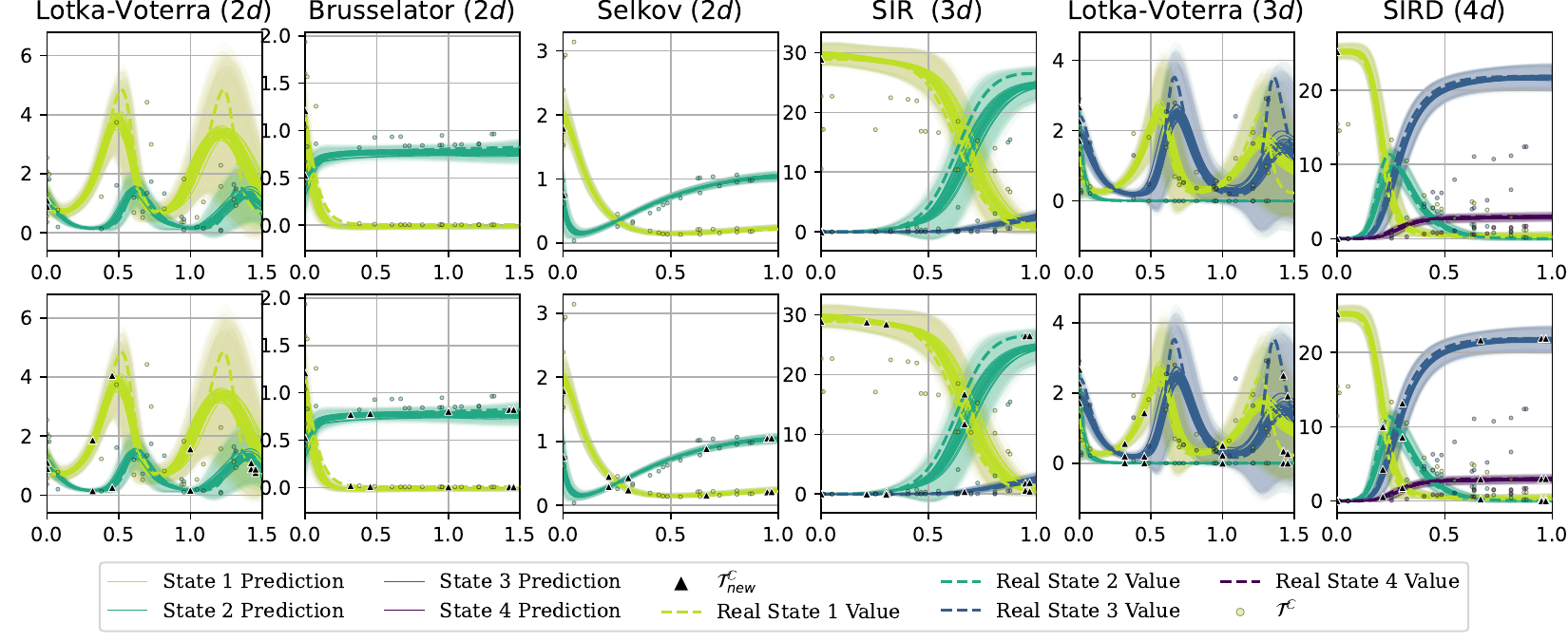}
    \caption{NODEP ($\lambda=0$) model performance on the test system in different meta-learning ODE problems.}
    \label{fig:nodep}
\end{figure}

\begin{figure}[h!]
    \centering
        \includegraphics[width=\textwidth]{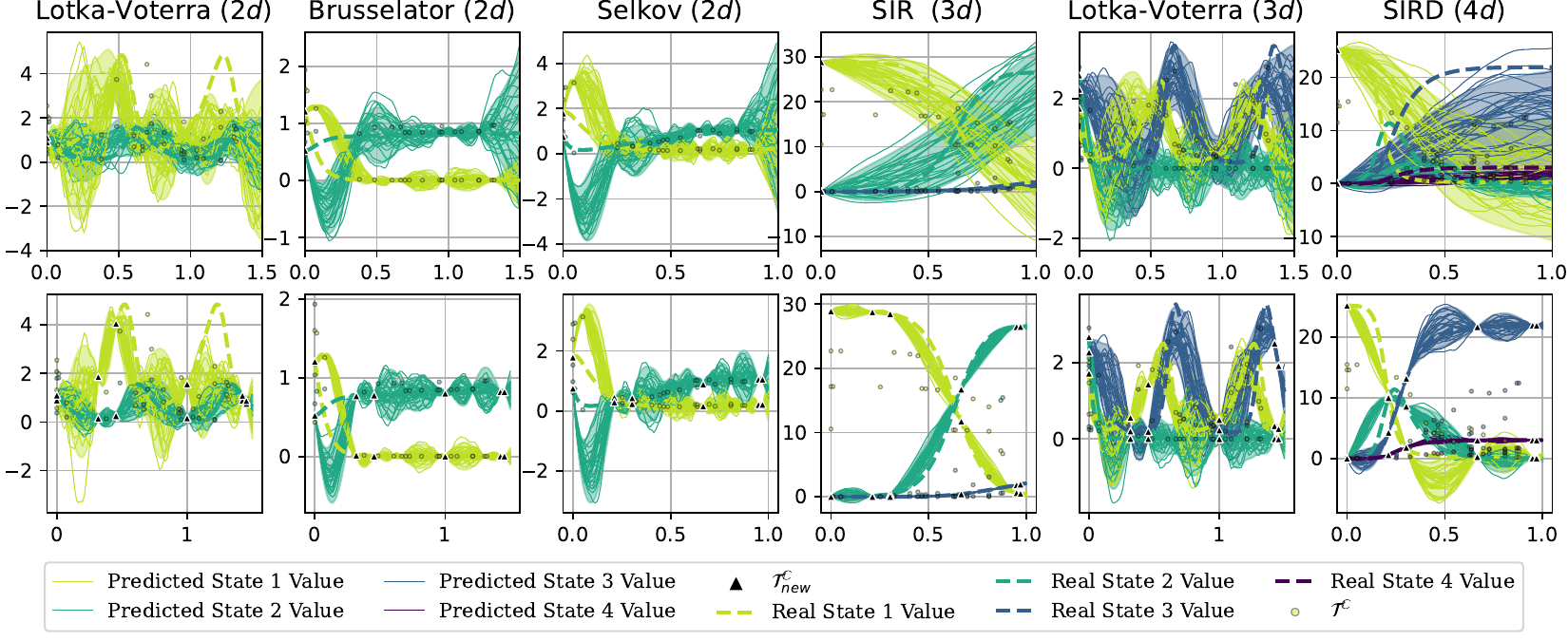}
    \caption{Performance of the GP model on the test system in different meta-learning ODE problems.}
    \label{fig:gp}
\end{figure}


\subsection{Optimization Problem Definition}
The definitions of the optimization problem are provided in Table.~\ref{tab:optimization_problem_formulations}.

\begin{table}[h!]
\centering
\caption{Optimization problem formulations}
\begin{adjustbox}{width=\textwidth,center}
\begin{tabular}{lcccccc}
    \toprule
    Problem & Lotka-Volterra ($2d$) & Brusselator ($2d$) & Selkov ($2d$) & SIR model ($3d$) & Lotka-Volterra ($3d$) & SIRD ($4d$) \\
    \midrule
    $g$ formulation &  $x_1$ & $x_1$ & $x_2$ & $x_1 / \sum_{i=1}^3x_i - 0.05 \sum_{i=1}^3x_i $ & $x_1$ &  $x_1 / \sum_{i=1}^3x_i - 0.05 \sum_{i=1}^3x_i $ \\
    Design space & $\makecell{\tau \in [0, 15], \\ \boldsymbol{x}_{dec}\in [0.1, 2.0]^2 \\ \boldsymbol{x}_{0} = \boldsymbol{x}_{dec}} $ &  $\makecell{\tau \in [0, 15], \\ \boldsymbol{x}_{dec}\in  [0.1, 2]^2 \\ \boldsymbol{x}_0 = \boldsymbol{x}_{dec}}$ &  $\makecell{\tau \in [0, 10], \\ \boldsymbol{x}_{dec}\in [0.1, 0.5]^2 \\ \boldsymbol{x}_{0} = \boldsymbol{x}_{dec}}$ & \makecell{$\tau \in [0, 1]$, \\ $x_{dec}\in [10, 30]$ \\ $\boldsymbol{x}_{0} = [x_{dec}, 0.1 \times  x_{dec}, 0]$} & \makecell{$\tau \in [0, 15]$, \\ $\boldsymbol{x}_{dec}\in [0.0, 2]^3 $ \\ $\boldsymbol{x}_{0} = \boldsymbol{x}_{dec}$} & \makecell{$\tau \in [0, 1]$\\ $x_{dec}\in [10, 30]$ \\ $\boldsymbol{x}_{0} = [x_{dec}, 0.1 \times  x_{dec}, 0, 0]$} \\
    \midrule 
    $\Delta t$ & $1.5$ & 1.5 & 1 & 0.1 & 1.5 & 0.1 \\
    \makecell{Optimization \\ ODE \\ definition} & \makecell{$\alpha=0.5$, \\ $\beta=1.2$, \\ $\delta = 1.0$, \\ $\gamma = 1.5$} & \makecell{$A=0.8$, \\ $B=1.5$} & \makecell{$a = 0.25$, \\ $b = 0.45$} & \makecell{$\beta = 1.5$ \\ $\gamma=5$} & \makecell{
    $ \alpha = 0.5 $ \\
    $ \beta = 1.2 $ \\
    $ \delta = 1.0 $ \\
    $ \gamma = 1.5 $ \\
    $ \epsilon = 0.5 $ \\
    $ \zeta = 1.2 $ \\
    $ \eta = 1.0 $ \\
    $ \theta = 1.5 $
}& \makecell{$\beta=1$\\$\gamma=0.5$\\$\mu=1$} \\
    \midrule
    Reference point & $[-1.771, 12.686]$ & $[-1.467, 3.887]$ & $[-0.474, 5.440]$ & $[0.51151, 0.79646]$ & $[-1.7557, 13.1687]$ & $[0.52198, 1.04]$ \\ 
    \bottomrule
\end{tabular}
\end{adjustbox}
\label{tab:optimization_problem_formulations}
\end{table}

\subsection{Inference Time} \label{App: complexity_analysis}

We provide the parameter counts, training time and prediction time comparison of SANODEP vs. baselines in Table.~\ref{tab: memory_comparison}, we note that all the models have the same latent dimensionality for a fair performance comparison, which is regarded as the key bottleneck of the latent variable typed meta-learning model \citep{kim2019attentive}. The parameter increase of SANODEP compared to NODEP is the additional encoder $\phi_r$ to handle the augmented state variable.

\begin{table}[h!]
    \centering
    \caption{Comparison of the number of parameters and prediction time of different models. We measure run time on the Lotka-Voterra ($d=2$) problem using an NVIDIA A40 GPU, We note training time and prediction time is per mini-batch, the optimization query time is per querying one observations.}
    \begin{adjustbox}{max width=\textwidth}
        \begin{tabular}{lcccc}
            \toprule
            Model & Number of Parameters & Training Time (\texttt{s}) & Prediction Time (\texttt{s}) & Optimization Query Time (\texttt{s})\\
            \midrule
            GP & NA & NA & NA & $23.81 \pm 43.94$\\
            NP & 21814 & $0.3016 \pm 0.0791$ & $0.2551 \pm 0.2337$ & $18.04 \pm 18.36$ \\
            NODEP & 30286 & $0.3792 \pm 0.1892$ & $0.3421 \pm 0.2412$ & $34.71 \pm 20.62$ \\
            SANODEP & 35514 & $0.3813 \pm 0.2423$ & $0.3482 \pm 0.2731$ & $39.47 \pm 37.30$ \\
            \bottomrule
        \end{tabular}
    \end{adjustbox}
    \label{tab: memory_comparison}
\end{table}

\section{Investigation of Prior Strength on Model Performance}
\subsection{Parameter estimation via PI-SANODEP} \label{App: detailes_of_pi_sanodep}

One advantage of SANODEP as a meta-learned dynamical model is its ability to estimate parameters. When the parametric form of the dynamical system is known, SANODEP can directly integrate this form into its governing ODE (Eq.~\ref{Eq: nodep_solve}), yielding a variant we call Physics-Informed SANODEP (PI-SANODEP).


For PI-SANODEP’s training loss function $\mathcal{L'}_{\theta}$, we denote the system parameters for each meta-training task as $\boldsymbol{u}_{\boldsymbol{f}}$, since these parameters are known during meta-training, we augment our loss (Eq.~\ref{Eq: ELBO_form}) by incorporating a likelihood term for these parameters, leading to the following expression:

\begin{equation}
\begin{aligned}
     & \text{log}\ p\left(\boldsymbol{X}_{new}^{\mathbb{T}}, \boldsymbol{u}_{\boldsymbol{f}} | \mathcal{T}^\mathbb{C}\cup \mathcal{T}_{new}^\mathbb{C}, \boldsymbol{T}_{new}^{\mathbb{T}} \right)  \\ & = \text{log}\ p\left(\boldsymbol{X}_{new}^{\mathbb{T}} | \boldsymbol{u}_{\boldsymbol{f}}, \mathcal{T}^\mathbb{C}\cup \mathcal{T}_{new}^\mathbb{C},\boldsymbol{T}_{new}^{\mathbb{T}} \right) + \text{log}\  p\left( \boldsymbol{u}_{\boldsymbol{f}} | \mathcal{T}^\mathbb{C}\cup \mathcal{T}_{new}^\mathbb{C}, \boldsymbol{T}_{new}^{\mathbb{T}} \right)
\end{aligned}
    \label{Eq: PI_ELBO_SANODEP}
\end{equation}
\noindent Since $\boldsymbol{u}_{\boldsymbol{f}}$ is known during meta-training, we use the encoder (described in Appendix.~\ref{App: model_structure}) as an inference network to calculate its likelihood $p\left( \boldsymbol{u}_{\boldsymbol{f}} | \mathcal{T}^\mathbb{C}\cup \mathcal{T}_{new}^\mathbb{C}, \boldsymbol{T}_{new}^{\mathbb{T}} \right)=\mathcal{N}\left(\boldsymbol{u}_{\boldsymbol{f}}; {\phi_{\mu}}_{sys}(\boldsymbol{h}_{sys}), diag\left({\phi_{\sigma}}_{sys}(\boldsymbol{h}_{sys})^2\right) \right)$. Regarding $p\left(\boldsymbol{X}_{new}^{\mathbb{T}} | \boldsymbol{u}_{\boldsymbol{f}}, \mathcal{T}^\mathbb{C}\cup \mathcal{T}_{new}^\mathbb{C},\boldsymbol{T}_{new}^{\mathbb{T}} \right)$, we however omit the dependence of $\boldsymbol{u}_{\boldsymbol{f}}$ and still make use of Eq.~\ref{Eq: ELBO_form} to calculate.  Thus, the final loss function for PI-SANODEP extends the original SANODEP loss with an added likelihood term $\text{log}\  p\left( \boldsymbol{u}_{\boldsymbol{f}} | \mathcal{T}^\mathbb{C}\cup \mathcal{T}_{new}^\mathbb{C}, \boldsymbol{T}_{new}^{\mathbb{T}}\right)$, capturing both the dynamics and parameter estimation.
For the model structure of PI-SANODEP, since $\boldsymbol{l}(t)= \boldsymbol{x}(t)$, hence ${\phi_{\mu}}_{dec}[\boldsymbol{x}(t), \boldsymbol{u}_{sys}, t]=\boldsymbol{x}(t)$.  the rest of the model structure is the same as SANODEP only except that the latent ODE $\boldsymbol{f}_{nn}$ becomes explicit.

\begin{figure}[h!]
    \centering
\includegraphics[width=\textwidth]{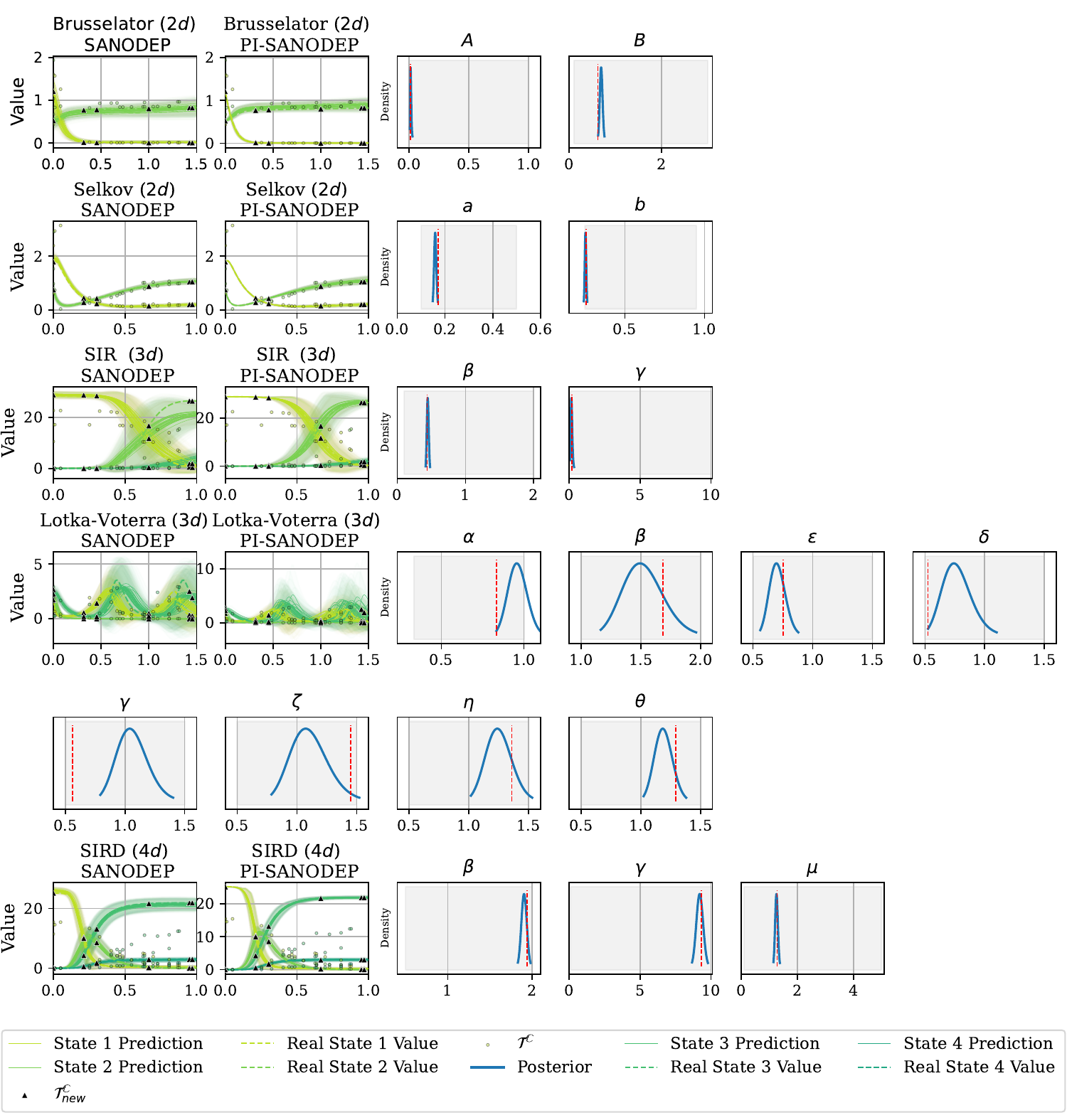}
    \caption{Additional experiments for meta-parameter estimation based on PI-SANODEP, the shaded area represents the meta-training parameter support $\mathcal{S}=\text{supp}(P)$. It can be seen that PI-SANODEP not only improves upon SANODEP over prediction (e.g., in the SIR example) but also provides a reasonable parameter estimation of the test system. 
    Parameter estimation tends to become difficult when the number of parameters is high as can been seen from the $3$-dimensional Lotka-Voterra system. }
    \label{fig:pi_sanodep_all_rest}
\end{figure}

We note that, while PI-SANODEP can improve itself toward the right estimation of parameters along the training process, as a non-local parameter estimation strategy, we empirically noticed that at the initial training stage of some dynamical systems, the inferred parameter distribution can be far from the correct ones and hence either trigger an exhaustive discretization step or have divergent trajectories. To mitigate such \textit{non-stable initialization} which causes numerical issues, given that we know the (interval) support $\mathcal{S}$ of our training system parameters $P$, we clip the sample parameter from amortized variation posterior $q\left(\boldsymbol{u}_{sys}| \mathcal{T}^\mathbb{C}\cup \mathcal{T}_{new}^{\mathbb{C}}(\mathbbm{1}_{forecast}) \cup \mathcal{T}_{new}^{\mathbb{T}}\right)$ by the support bound (provided in the Appendix.\ref{sec: training_data_describe}), which significantly improves the stability of PI-SANODEP training.

We report the PI-SANODEP model performance in comparison with SANODEP, together with its parameter estimation on the rest of the problem in Fig.~\ref{fig:pi_sanodep_all_rest}. For consistency, we use log-normal for all parameters in different problems. We conclude by remarking that the parameter estimation functionality of PI-SANODEP is standalone besides the optimization application.

\subsection{Vector Valued GP Prior as Task Distribution}\label{app: vec_gp_prior}
\begin{figure}[h!]
    \centering
\includegraphics[width=\textwidth]{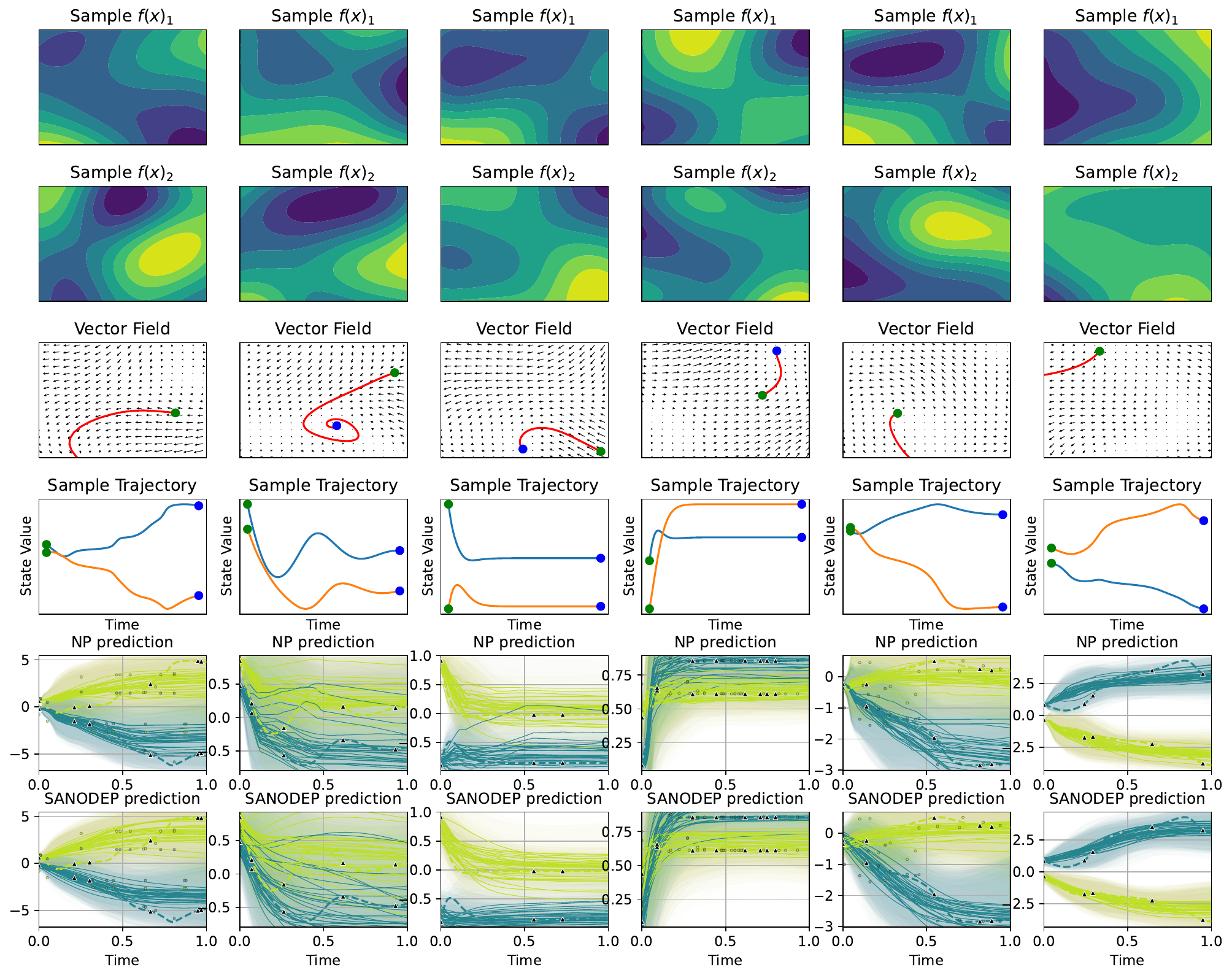}
    \caption{Task distribution generated by vector-valued Gaussian Process (GP) priors. The first two rows display samples from parametric GP priors. The third and fourth rows show the sampled vector fields and the resulting trajectories starting from sampled initial conditions. The last two rows depict the meta-model predictions on these trajectories. It can be seen that while the GP-based vector field has introduced a very flexible task distribution, both meta-learn models tend to underfit on these trajectories.}    \label{fig:gp_ode}
\end{figure}

\begin{figure}[h!]
    \centering
\includegraphics[width=\textwidth]{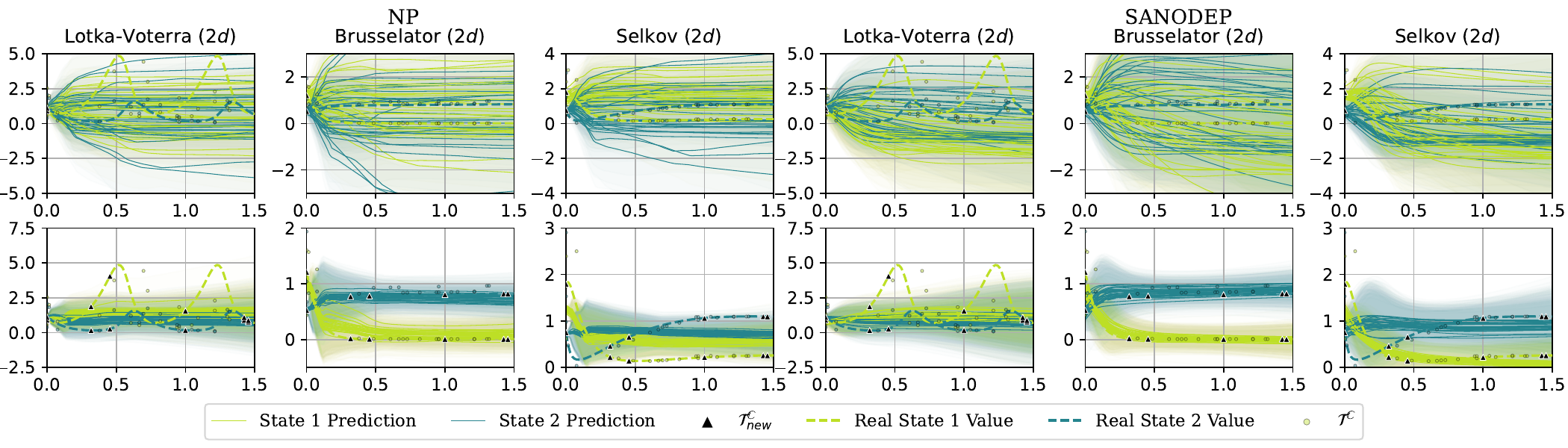}
    \caption{Cross domain generalization exploration: SANODEP and NP meta trained on GP-ODE seems to be able to capture the trajectories from Bruseelator system.}
    \label{fig:cross_domain_generalization}
\end{figure}
As an initial investigation with SANODEP, inspired by the recent approach of using Gaussian Processes (GPs) to model vector fields (e.g., \cite{heinonen2018learning}), we propose to use vector-valued GP as the non-parametric prior for the vector field, denoted by $\boldsymbol{f} \sim \mathcal{GP}(\boldsymbol{0}, K(\boldsymbol{x}, \boldsymbol{x}'))$, as task distributions for meta-learning for dynamical systems. Further details on the data used for meta-training are provided in Appendix~\ref{sec: training_data_describe}. The dynamical systems sampled from this vector-valued GP prior, as illustrated in Fig.~\ref{fig:gp_ode}, demonstrate its ability to represent a broad distribution of dynamical systems. Additionally, we offer a visual comparison of SANODEP and NP on the sampled trajectories. This comparison reveals a trade-off between fitting accuracy and trajectory flexibility for both models. Specifically, both models tend to significantly underfit trajectories that involve oscillations, indicating the noticable complicity of this meta-learning problem. 

We also explore whether this flexible task distribution facilitates cross-domain generalization. To this end, we employ the trained model on different dynamical systems that share the same state variable dimensionality, as shown in Fig.~\ref{fig:cross_domain_generalization}. Both meta-learned models demonstrate some potential for cross-domain generalization particularly from the Brusselator system. However, generalization to systems involving oscillatory behaviors proves to be especially challenging.

\section{Notations}\label{Sec:notation}
\begin{longtable}{ll}
\caption{Nomenclature Table} \label{tab:numemclature_table} \\ 
\toprule
Notation & Meaning \\
\midrule
\endfirsthead

\caption[]{Nomenclature Table (Continued)} \\
\toprule
Notation & Meaning \\
\midrule
\endhead

\bottomrule
\endfoot

\bottomrule
\endlastfoot

$\boldsymbol{x}$ & State variable vector \\
$d_{\boldsymbol{x}}$ & Dimension of the state space \\
$t_0$ & Initial time for the dynamical system \\
$t_i$ & Trajectory time samples, sampled irregularly \\
$t$ & Termination time for a trajectory \\
$\mathbb{C}$ & Context Sets \\
$\boldsymbol{f}$ & Vector field governing the dynamics \\
$g(\cdot)$ & Aggregation function used in optimization \\
$L_0$ & Latent initial condition \\
$L_D$ & Dynamics of the system \\
$\phi_r$ & Context encoding \\
$\boldsymbol{r}$ & Context representation vector \\
$\boldsymbol{l}_0$ & Realizations of latent initial conditions \\
$\boldsymbol{l}(t)$ & Latent states at time $t$ \\
$\boldsymbol{u}$ & Realizations of the control term representing latent dynamics \\
$\boldsymbol{f}_{nn}$ & Neural network parameterized vector field \\
$\theta_{ode}$ & Parameterisation of the ODE system \\
$\boldsymbol{f}_{evolve}$ & Numerical solver implementation of the vector field \\
$h$ & Instantenous time variable in continuous dynamics \\
$d_{\boldsymbol{l}}$ & Dimension of the latent state space \\
$\mathbb{T}$ & Target Sets \\
$\boldsymbol{T}^{\mathbb{T}}$ & Target times for predictions \\
$\boldsymbol{X}^{\mathbb{T}}$ & Target state values \\
$\mathcal{F}$ & Distribution of dynamical systems \\
$P$ & Parametric Task distribution, often stochastic \\
$\mathcal{X}_{\boldsymbol{x}}$ & State variable space \\
$\mathcal{X}_0$ & Initial condition space \\
$\tau$ & Time domain for the system evolution \\
$M$ & Number of context trajectories \\
$N_l$ & Number of context elements in trajectory $l$ \\
$\mathcal{T}^{\mathbb{C}}$ & Context observations from $M$ trajectories \\
$\mathcal{T}_l^{\mathbb{C}}$ & Context observations from the $l$th trajectory \\
$\mathcal{T}_{new}$ & A new trajectory which augments the context set $\mathbb{C}$ \\
$\mathcal{T}_{new}^\mathbb{C}$ & Context observations from the new trajectory \\
$\mathcal{T}_{new}^\mathbb{T}$ & Target observations from the new trajectory \\
$D_{sys}$ & Latent variable conditioned on $M+1$ trajectories \\
$\boldsymbol{u}_{sys}$ & Realization of system-aware latent dynamics \\
$\boldsymbol{r}_{sys}$ & System context representation vector \\
${\phi_r}_{sys}$ & Context encoding augmented with the initial state \\
${\boldsymbol{X}_{new}^{\mathbb{T}}}$ & Target state values for the new trajectory \\
${\boldsymbol{T}_{new}^{\mathbb{T}}}$ & Target times for the new trajectory \\
${{\boldsymbol{x}_{new}}_{i}^{\mathbb{C}}}$ & State vector samples from the new trajectory \\
${t_{new}}_i^{\mathbb{C}}$ & Time samples from the new trajectory \\
$K$ & Number of observations in the new trajectory \\
$\mathcal{L}_{\theta}$ & SANODEP's multi scenario loss function parametrised by $\theta$ \\
$p_{\theta}$ & SANODEP prediction parametrised by $\theta$ \\
$\mathbbm{1}_{forecast}$ & Bernoulli Indicator \\
$\lambda$ & Parameter for Bernoulli distribution in decision-making process \\
$\mathcal{N}_{opt}$ & Search space of the number of observations in the trajectory \\
$\Delta t$ & Minimum delay between observations \\
$t_{max}$ & Maximum time observation \\
$N_{max}$ & Maximum number of observations per trajectory \\
$\alpha(\cdot)$ & Batch acquisition function \\
\texttt{HVI} & Hypervolume improvement based on the Pareto frontier \\
$\mathcal{F}*$ & Pareto frontier \\
$M_{min}$ & Minimum number of context trajectories observed in a system \\
$M_{max}$ & Maximum number of context trajectories observed in a system \\
$N_{\boldsymbol{x}_0}$ & Sampled number of initial conditions within a system \\
$N_{sys}$ & Number of dynamical system samples per training iteration \\
$N_{grid}$ & Number of points in the time grid for trajectory evaluation \\
$m_{min}, m_{max}$ & Min and max context points within a trajectory \\
$n_{min}, n_{max}$ & Min and max target points beyond context in a trajectory \\
$\mathcal{S}$ &  The support set of the stochastic parameters for meta-task distribution \\
$\boldsymbol{T}_{grid}$ & Time grid used for trajectory evaluations \\
\end{longtable}

\end{document}